\documentclass{article}

\usepackage[utf8]{inputenc} 
\usepackage[T1]{fontenc}    
\usepackage{hyperref}       
\usepackage{url}            
\usepackage{booktabs}       
\usepackage{amsfonts}       
\usepackage{nicefrac}       
\usepackage{microtype}      
\usepackage{lipsum}
\usepackage{natbib}
\usepackage{fancyhdr}       
\usepackage{graphicx}       
\usepackage{wrapfig}
\usepackage{amsmath}
\usepackage{amssymb}
\usepackage{subcaption}
\usepackage{caption}
\usepackage{tabularx}
\usepackage{float}
\usepackage{placeins}
\usepackage{authblk}        
\usepackage{setspace}       
\usepackage{algorithm}
\usepackage{algpseudocode}
\usepackage{amsthm}
\newtheorem{proposition}{Proposition}
\newtheorem{lemma}{Lemma}

\newtheorem{corollary}{Corollary}

\title{\LARGE \bf }

\author[1]{Kumbha Nagaswetha}
\author[2]{Rabi Pathak} 

\affil[1]{ \texttt{nagaswethak@iisc.ac.in}}
\affil[2]{ \texttt{rabipathak@iisc.ac.in}} 
\date{May 4th 2026}
\title{Efficient Conditioning: Why Pseudo-Observation
Batch Bayesian Optimization Works (When It Doesn’t)}

\begin{document}

\maketitle

\begin{abstract}
Constant Liar (CL), Kriging Believer (KB), and fantasy models are widely used for batch selection in parallel Bayesian Optimization, yet a unified theory explaining their effectiveness and conditions under which they fail has been lacking. We identify \emph{efficient conditioning} as the key surrogate property: the ability to update predictions in closed form when data is augmented. We prove that Gaussian Processes satisfy this requirement, producing provably distinct batch points with separation of order $\ell$, and that this holds for any acquisition function monotonically non-decreasing in posterior uncertainty (EI, UCB, PI), with qualitatively similar behavior for Thompson Sampling. We unify CL, KB, and fantasy models as instances of a single conditioning mechanism differing only in the lie-value distribution, and draw quantitative connections to Local Penalization (LP) and qualitative connections to Determinantal Point Processes (DPPs). To disentangle model structure from optimizer randomness, we introduce the \emph{Structural Diversity Diagnostic} (SDD), a reusable methodology for testing surrogate compatibility. Experiments on Hartmann-6D, Ackley-8D, Levy-10D, and SVM hyperparameter tuning validate all theoretical predictions: CL/KB's implicit penalty matches or outperforms explicit LP; greedy conditioning achieves convergence on par with joint $q$-EI; efficient conditioning extends to Multiquadric RBF networks; and parametric surrogates produce degenerate batches even when fully retrained (random forests), while neural networks regain diversity only at ${\approx}15\times$ the wall-clock cost of GP conditioning. Robustness is confirmed across multiple initial datasets and under observation noise.
\end{abstract}

\section{Introduction}
\label{sec:intro}

Bayesian Optimization (BO) is a leading framework for optimizing expensive black-box functions \cite{snoek2012practical, lookman2019active, forrester2008engineering}. Standard BO operates sequentially; batch BO selects $q > 1$ points simultaneously \cite{desautels2014parallelizing, wang2018batched}, reducing wall-clock time by orders of magnitude. However, naively selecting the top-$q$ acquisition maximizers yields nearly identical candidates.

Among the most widely adopted solutions are \emph{Constant Liar} (CL) and \emph{Kriging Believer} (KB) \cite{ginsbourger2010kriging, ginsbourger2008multi}, which construct batches sequentially through pseudo-observations: after selecting $\mathbf{x}_j$, a synthetic observation $(\mathbf{x}_j, \tilde{y}_j)$ is added, the surrogate is updated, and the acquisition function re-optimized. Despite their simplicity, CL and KB are implemented in major frameworks \cite{balandat2020botorch} and perform well empirically. Yet they work reliably only with Gaussian Process (GP) surrogates; with alternatives such as neural networks, batch points can collapse to identical locations, as we demonstrate empirically (Section~\ref{sec:experiments}).

\emph{What property of a surrogate model determines compatibility with pseudo-observation batch selection?} This question grows in importance as BO scales beyond standard GPs \cite{snoek2015scalable, hutter2011sequential, hensman2013gaussian}; existing analyses assume GP surrogates \cite{ginsbourger2010kriging, desautels2014parallelizing, contal2013parallel} and do not address surrogate-level compatibility with pseudo-observation batch selection. This paper provides the first formal characterization, with explicit propositions, of surrogate model requirements for pseudo-observation batch selection:
\begin{itemize}
\item We identify \emph{efficient conditioning}, the ability to update predictions in closed form when data is augmented, as the key surrogate property for batch diversity, and prove that GP surrogates satisfy it while parametric models do not (Propositions~\ref{prop:no_duplicate}--\ref{prop:suppression_radius}, Corollary~\ref{cor:parametric_degeneracy}).
\item We prove that the diversity mechanism is \emph{acquisition-function-agnostic}, covering EI, UCB, and PI (Proposition~\ref{prop:acq_general}), with qualitatively similar behavior for Thompson Sampling.
\item We unify CL, KB, and fantasy models as instances of a single conditioning mechanism, establish a formal connection to Local Penalization (Eq.~\ref{eq:implicit_form}), and draw a qualitative bridge to Determinantal Point Processes.
\item We verify the theory empirically via the \emph{Structural Diversity Diagnostic} (SDD), which eliminates optimizer randomness to confirm that batch diversity is a structural surrogate property.
\item We show efficient conditioning extends to MQ-RBF networks, and that parametric surrogates produce degenerate batches even when fully retrained.
\item We validate all predictions on Hartmann-6D, Ackley-8D, Levy-10D, and SVM tuning, demonstrating that greedy CL/KB matches or outperforms LP and is competitive with joint $q$-EI within noise.
\end{itemize}

\section{Background}
\label{sec:background}

We review GP regression and pseudo-observation strategies for batch construction; full derivations are in Appendix~\ref{app:extended_background}.

\subsection{Gaussian Process Regression}

We consider minimization of a black-box function $f: \mathcal{X} \rightarrow \mathbb{R}$. A GP defines a distribution over such functions, specified by a mean function $m(\mathbf{x})$ and covariance kernel $k(\mathbf{x}, \mathbf{x}')$ \cite{rasmussen2006gaussian}. Given dataset $\mathcal{D} = \{(\mathbf{x}_i, y_i)\}_{i=1}^n$, the posterior predictive distribution at test point $\mathbf{x}_*$ is Gaussian with mean $\mu(\mathbf{x}_*) = m(\mathbf{x}_*) + \mathbf{k}_*^\top \boldsymbol{\alpha}$ and variance $\sigma^2(\mathbf{x}_*) = k(\mathbf{x}_*, \mathbf{x}_*) - \mathbf{k}_*^\top \mathbf{K}_y^{-1} \mathbf{k}_*$, where $\mathbf{K}_y = \mathbf{K} + \sigma_n^2\mathbf{I}$ and $\mathbf{K}_y\boldsymbol{\alpha} = \mathbf{y} - \mathbf{m}$. Both quantities depend on the full dataset through the inverse of $\mathbf{K}_y$; modifying the dataset alters predictions globally.

\subsection{Sequential Batch Selection}

Batch BO selects $q > 1$ points for simultaneous evaluation. CL \cite{ginsbourger2010kriging} and KB \cite{ginsbourger2008multi} build batches by iteratively augmenting the dataset with pseudo-observations and re-optimizing the acquisition function (Algorithm~\ref{alg:batch}). In CL, the lie value $\tilde{y}_j$ is a fixed constant (min, max, or mean of observed values); in KB, $\tilde{y}_j = \mu_{\mathcal{D}_j}(\mathbf{x}_j)$, the current posterior mean --- producing exactly zero mean shift, so diversity arises purely from variance reduction (formalized in Lemma~\ref{prop:kb_invariance}). Algorithm~\ref{alg:batch} is stated in surrogate-agnostic form; Section~\ref{sec:requirements} analyzes which surrogates make this loop produce diverse batches.

\begin{algorithm}[t]
\caption{Sequential Batch Selection via Pseudo-Observations}
\label{alg:batch}
\begin{algorithmic}[1]
\Require Data $\mathcal{D}$, batch size $q$, lie strategy
\Ensure Batch $\mathcal{B}$
\State $\mathcal{D}' \gets \mathcal{D}$, $\mathcal{B} \gets \emptyset$
\For{$j = 1$ to $q$}
    \State Fit surrogate on $\mathcal{D}'$ to obtain $\mu(\mathbf{x}), \sigma(\mathbf{x})$
    \State $\mathbf{x}_j \in \arg\max_{\mathbf{x}} \alpha(\mathbf{x} \mid \mathcal{D}')$
    \State $\mathcal{B} \gets \mathcal{B} \cup \{\mathbf{x}_j\}$
    \State Assign pseudo-observation $\tilde{y}_j$ based on strategy
    \State $\mathcal{D}' \gets \mathcal{D}' \cup \{(\mathbf{x}_j, \tilde{y}_j)\}$
\EndFor
\State \Return $\mathcal{B}$
\end{algorithmic}
\end{algorithm}
\noindent While Algorithm~\ref{alg:batch} makes no assumptions beyond the ability to provide $\mu(\mathbf{x})$ and $\sigma(\mathbf{x})$, pseudo-observation strategies work reliably only with GPs \cite{gonzalez2016batch, wang2018batched}; with parametric alternatives, batch points collapse to identical locations. Characterizing the structural properties responsible for this dichotomy is the focus of the following section.

\section{Surrogate Model Requirements for Batch Selection}
\label{sec:requirements}

This section addresses the central question: what structural property of a surrogate model determines its compatibility with pseudo-observation batch selection? \textbf{Notation:} $\mathbf{x}_*$ denotes a generic conditioning point; $\mathbf{x}_j$ ($j = 1, \ldots, q$) indexes batch members; $\mathcal{D}' = \mathcal{D} \cup \{(\mathbf{x}_*, \tilde{y})\}$ is the augmented dataset; $\mathcal{D}_j$ denotes the dataset after $j{-}1$ conditioning steps.

\subsection{The Efficient Conditioning Requirement}

Algorithm~\ref{alg:batch} iteratively adds pseudo-observations and re-optimizes the acquisition function. For this to succeed, predictions must change when data is added, and updates must be computationally efficient. We formalize this:

\begin{quote}A surrogate supports \emph{efficient conditioning} if, when $\mathcal{D}$ is augmented with $(\mathbf{x}_*, \tilde{y})$, updated predictions $\mu'(\mathbf{x}), \sigma'^2(\mathbf{x})$ can be computed via closed-form or fast linear algebra, without iterative parameter optimization.
\end{quote}

\subsection{Gaussian Processes}

GPs satisfy efficient conditioning naturally. When pseudo-observation $(\mathbf{x}_*, \tilde{y})$ is added, the updated predictions follow from standard conditioning identities \cite{rasmussen2006gaussian}:
\begin{align}
\mu_{\mathcal{D}'}(\mathbf{x}) &= \mu_{\mathcal{D}}(\mathbf{x}) + w(\mathbf{x}, \mathbf{x}_*) \big(\tilde{y} - \mu_{\mathcal{D}}(\mathbf{x}_*)\big), \label{eq:mean_update} \\
\sigma^2_{\mathcal{D}'}(\mathbf{x}) &= \sigma^2_{\mathcal{D}}(\mathbf{x}) - \frac{k(\mathbf{x}, \mathbf{x}_*)^2}{\sigma^2_{\mathcal{D}}(\mathbf{x}_*) + \sigma_n^2}, \label{eq:var_update}
\end{align}
where $w(\mathbf{x}, \mathbf{x}_*) = k(\mathbf{x}, \mathbf{x}_*) / (\sigma^2_{\mathcal{D}}(\mathbf{x}_*) + \sigma_n^2)$. With slight abuse of notation, $k(\mathbf{x}, \mathbf{x}_*)$ in Eq.~\ref{eq:var_update} denotes the posterior covariance under $\mathcal{D}$, $k_{\mathcal{D}}(\mathbf{x}, \mathbf{x}_*) = k(\mathbf{x}, \mathbf{x}_*) - \mathbf{k}(\mathbf{x}, \mathbf{X})^\top \mathbf{K}_y^{-1} \mathbf{k}(\mathbf{X}, \mathbf{x}_*)$; this reduces to $\sigma^2_{\mathcal{D}}(\mathbf{x}_*)$ at $\mathbf{x} = \mathbf{x}_*$ and to the prior kernel when $\mathbf{x}_*$ is far from training data \cite{rasmussen2006gaussian}. The mean shifts proportionally to the discrepancy $\tilde{y} - \mu_{\mathcal{D}}(\mathbf{x}_*)$, weighted by kernel correlation. The variance decreases everywhere, with the largest reduction near $\mathbf{x}_*$. This update requires no hyperparameter re-optimization; it is purely algebraic, computable via rank-one Cholesky updates in $O(n^2)$ per step, or $O(qn^2 + n^3)$ amortized over the full batch. The variance reduction (Eq.~\ref{eq:var_update}) is the key mechanism: it reduces acquisition values near $\mathbf{x}_*$ (since acquisition functions depend on $\sigma$), creating implicit repulsion between batch members without explicit diversity penalties.

\subsection{Non-GP Surrogates}

\textbf{Neural networks} are parametric models whose predictions depend on fixed weights $\theta$; adding $(\mathbf{x}_*, \tilde{y})$ without retraining leaves predictions unchanged ($\alpha(\mathbf{x} \mid \mathcal{D}') = \alpha(\mathbf{x} \mid \mathcal{D})$), yielding $\mathbf{x}_1 = \cdots = \mathbf{x}_q$. Retraining is prohibitively expensive ($q$ full training cycles per batch) and does not provide calibrated uncertainty. \textbf{Random forests} cannot be updated in closed form because each tree's structure depends globally on the dataset; even when rebuilt, bootstrap sampling dilutes one pseudo-observation across 200 trees. \textbf{RBF networks} with \emph{data-dependent centers} ($\mathbf{c}_i = \mathbf{x}_i$) satisfy efficient conditioning: the interpolation system $\mathbf{\Phi}\mathbf{w} = \mathbf{y}$ expands analogously to a GP's kernel matrix. We use the power function as the $\sigma$ proxy; this is a deterministic worst-case bound rather than a probabilistic posterior, but it suffices for the structural diversity property (Appendix~\ref{app:mqrbf_uncertainty}). With \emph{fixed centers}, only the weights update, placing these models in the parametric category. Table~\ref{tab:surrogate_compatibility} summarizes compatibility; additional analysis of SVMs, sparse GPs, and lie strategy variants is in Appendix~\ref{app:surrogate_analysis}.

\begin{table}[t]
\centering
\caption{Surrogate model compatibility with pseudo-observation batch selection.}
\label{tab:surrogate_compatibility}
\small
\begin{tabular}{lcc}
\toprule
\textbf{Surrogate Model} & \textbf{Efficient Conditioning} & \textbf{CL/KB Compatible} \\
\midrule
Gaussian Process & Yes & Yes \\
MQ-RBF (data centers) & Yes & Yes \\
RBF (fixed centers) & No & No \\
Neural Network & No & No \\
Last-layer Bayesian NN$^\dagger$ & Partial & Limited \\
Random Forest & No & No \\
SVM Regression & No & No \\
Sparse GP (fixed inducing) & Partial & Limited \\
\bottomrule
\multicolumn{3}{l}{\footnotesize $^\dagger$Closed-form Bayesian linear-regression updates on the last layer satisfy}\\
\multicolumn{3}{l}{\footnotesize efficient conditioning given a fixed feature map; the feature map itself is not updated.}
\end{tabular}
\end{table}

\subsection{Formal Diversity Guarantees}
\label{sec:formal_guarantees}

We formalize the diversity properties that emerge from efficient conditioning (see Figure~\ref{fig:overview} for a visual summary).

\begin{proposition}
\label{prop:no_duplicate}
Let the surrogate be a GP with a strictly positive-definite kernel $k$ and noise variance $\sigma_n^2 > 0$, and let $\alpha$ be monotonically non-decreasing in $\sigma$ for fixed $\mu$. Assume that, after each conditioning step, $\alpha(\cdot \mid \mathcal{D}')$ attains its maximum (guaranteed by compactness of $\mathcal{X}$). Then Algorithm~\ref{alg:batch} with CL or KB produces distinct batch points: $\mathbf{x}_i \neq \mathbf{x}_j$ for all $i \neq j$. If $\alpha(\cdot \mid \mathcal{D}')$ has a unique maximizer, the next point is uniquely determined; otherwise, any maximizer lies outside the suppression neighborhood of all previous points.
\end{proposition}

\noindent\textit{Proof sketch.} After conditioning on $(\mathbf{x}_1, \tilde{y}_1)$, variance at $\mathbf{x}_1$ strictly decreases (Eq.~\ref{eq:var_update}). Since $\alpha$ is non-decreasing in $\sigma$, this reduces $\alpha$ at $\mathbf{x}_1$. By continuity, suppression extends to a neighborhood, so $\mathbf{x}_2 \neq \mathbf{x}_1$. The argument applies inductively. Full proof in Appendix~\ref{app:proofs}. $\square$

\begin{proposition}[Variance Suppression Radius]
\label{prop:suppression_radius}
Consider a GP with SE kernel $k(\mathbf{x}, \mathbf{x}') = \sigma_f^2 \exp(-\|\mathbf{x} - \mathbf{x}'\|^2 / (2\ell^2))$,\footnotemark{} in the prior-dominated regime where $\sigma^2_{\mathcal{D}}(\mathbf{x}_*) \approx \sigma_f^2$, so the posterior covariance $k_{\mathcal{D}}(\mathbf{x}, \mathbf{x}_*)$ is well-approximated by the prior kernel $k(\mathbf{x}, \mathbf{x}_*)$. Assume $\mathbf{x}_*$ is a local maximizer of posterior variance, so $\sigma^2_{\mathcal{D}}(\mathbf{x}) \leq \sigma^2_{\mathcal{D}}(\mathbf{x}_*)$ for $\|\mathbf{x} - \mathbf{x}_*\| \leq r$. After conditioning on $(\mathbf{x}_*, \tilde{y})$, the suppression zone radius satisfies
\begin{equation}
r_\tau(\mathbf{x}_*) \geq \ell \sqrt{\log\!\left(\frac{\sigma_f^4}{\tau \, \sigma^2_{\mathcal{D}}(\mathbf{x}_*) \left(\sigma^2_{\mathcal{D}}(\mathbf{x}_*) + \sigma_n^2\right)}\right)}.
\label{eq:suppression_radius}
\end{equation}
Thus, batch points are separated by at least $\Omega(\ell)$.
\end{proposition}
\footnotetext{Stated for the SE kernel for tractability. The qualitative result holds for any stationary kernel with monotonically decreasing correlation; our experiments use Mat\'{e}rn $\nu{=}2.5$.}

\noindent\textit{Proof sketch.} Substitute the SE kernel into Eq.~\ref{eq:var_update} to obtain fractional variance reduction $\propto \exp(-\|\mathbf{x} - \mathbf{x}_*\|^2/\ell^2)$. Setting this $\geq \tau$ and solving for $\|\mathbf{x} - \mathbf{x}_*\|$ yields Eq.~\ref{eq:suppression_radius}. Full derivation in Appendix~\ref{app:proofs}. $\square$

\begin{lemma}[KB Mean Invariance]
\label{prop:kb_invariance}
Under KB ($\tilde{y} = \mu_{\mathcal{D}}(\mathbf{x}_*)$), $\mu_{\mathcal{D}'}(\mathbf{x}) = \mu_{\mathcal{D}}(\mathbf{x})$ for all $\mathbf{x}$. Batch diversity under KB arises \emph{purely} from variance reduction.
\end{lemma}

\begin{proof}
From Eq.~\ref{eq:mean_update}, setting $\tilde{y} = \mu_{\mathcal{D}}(\mathbf{x}_*)$ makes the correction term vanish. The mean is unchanged; acquisition reshaping occurs exclusively through Eq.~\ref{eq:var_update}.
\end{proof}

\begin{corollary}[Parametric Model Degeneracy]
\label{cor:parametric_degeneracy}
For any surrogate whose predictions are independent of the training set after parameter fitting, Algorithm~\ref{alg:batch} with a deterministic optimizer produces $\mathbf{x}_1 = \mathbf{x}_2 = \cdots = \mathbf{x}_q$.
\end{corollary}

\begin{proposition}[Acquisition Function Generality]
\label{prop:acq_general}
Let $\alpha(\mu, \sigma)$ be any acquisition function monotonically non-decreasing in $\sigma$ for fixed $\mu$. Then Propositions~\ref{prop:no_duplicate}--\ref{prop:suppression_radius} hold with $\alpha$ in place of EI. This covers EI, UCB, and PI.
\end{proposition}

\noindent\textit{Proof sketch.} The only property of EI used is monotonicity in $\sigma$. The strict variance reduction at $\mathbf{x}_1$ implies $\alpha$ decreases there for any such $\alpha$. See Appendix~\ref{app:proofs}. $\square$

\noindent\textbf{Remark (Thompson Sampling).} TS does not fit the $\alpha(\mu,\sigma)$ template since the posterior sample depends on the full covariance. Conditioning still shifts the TS argmax via reduced variance, but the formal $\Omega(\ell)$ separation guarantee does not transfer without further assumptions on the sample paths.

\subsection{Connection to Local Penalization}
\label{sec:local_penalization}

Local Penalization (LP) \cite{gonzalez2016batch} modifies the acquisition function with explicit repulsion: $\alpha_{\text{LP}}(\mathbf{x}) = \alpha(\mathbf{x}) \prod_j \varphi(\mathbf{x}; \mathbf{x}_j)$, where $\varphi \in [0,1]$ decreases near $\mathbf{x}_j$. We show that CL/KB implements \emph{implicit} LP through the GP posterior. Under KB (where the mean is unchanged by Lemma~\ref{prop:kb_invariance}), the ratio of updated to original EI serves as an implicit penalizer:
\begin{equation}
\psi(\mathbf{x}; \mathbf{x}_*) \triangleq \frac{\text{EI}_{\mathcal{D}'}(\mathbf{x})}{\text{EI}_{\mathcal{D}}(\mathbf{x})} \approx \frac{\sigma'(\mathbf{x})}{\sigma(\mathbf{x})} = \sqrt{1 - \frac{\sigma_f^4 \exp(-\|\mathbf{x} - \mathbf{x}_*\|^2 / \ell^2)}{\sigma^2(\mathbf{x})(\sigma^2(\mathbf{x}_*) + \sigma_n^2)}},
\label{eq:implicit_form}
\end{equation}
where the approximation $\text{EI}'/\text{EI} \approx \sigma'/\sigma$ is tight in the exploration-dominated regime ($\sigma\phi(Z) \gg (f^+ - \mu)\Phi(Z)$); in the exploitation regime the conditioned mean update (under CL) already pushes the acquisition away from $\mathbf{x}_*$. The penalty $\psi$ matches the profile of an LP penalizer: near $0$ close to $\mathbf{x}_*$, approaching $1$ far away. For CL, an additional mean-shift term gives qualitatively similar repulsion (Appendix~\ref{app:surrogate_analysis}). The key difference is that CL/KB inherits the kernel's smoothness implicitly, while LP requires explicit penalty shape and Lipschitz constant specification. Section~\ref{sec:lp_experiment} validates this empirically.

\subsection{Regret Implications}
\label{sec:regret}

The $\Omega(\ell)$-scale separation (Proposition~\ref{prop:suppression_radius}) has direct consequences for optimization performance. Batch GP-UCB achieves sublinear regret when the information hallucinated within a batch is bounded, which in turn requires diverse exploration within the batch~\cite{desautels2014parallelizing}, and a similar result holds for UCB-PE~\cite{contal2013parallel}. Our guarantee ensures exactly this: each pseudo-observation suppresses a region of radius $\Omega(\ell)$, preventing the pathological clustering that drives regret blowup. For GP surrogates, this connects our structural result directly to the diversity conditions used in existing sublinear regret bounds. Extending the regret analysis to other surrogates satisfying efficient conditioning is a natural next step.

\subsection{Unification of Pseudo-Observation Strategies}
\label{sec:unification}

CL, KB, and \emph{fantasy models} \cite{snoek2012practical, wilson2018maximizing, balandat2020botorch} are all instances of a single conditioning mechanism differing only in the lie value $\tilde{y}$: CL uses a fixed constant, KB uses $\mu_{\mathcal{D}}(\mathbf{x}_*)$, and fantasy models sample $\tilde{y} \sim \mathcal{N}(\mu_{\mathcal{D}}(\mathbf{x}_*), \sigma^2_{\mathcal{D}}(\mathbf{x}_*) + \sigma_n^2)$. Since the variance update (Eq.~\ref{eq:var_update}) is \emph{independent of} $\tilde{y}$, all three produce identical suppression geometry, differing only in mean shift: zero for KB (Lemma~\ref{prop:kb_invariance}), deterministic for CL, stochastic for fantasy models. CL/KB and DPP-based methods \cite{kathuria2016batched} both implement kernel-scaled repulsion, though through different machinery (sequential conditioning vs.\ probabilistic subset selection; see Appendix~\ref{app:surrogate_analysis}), placing CL/KB, LP, and DPPs on a shared conceptual spectrum of repulsion-based batch diversity.

\section{Experiments}
\label{sec:experiments}

We validate the theoretical analysis through experiments on synthetic benchmarks and a real-world task. Diversity values are mean pairwise Euclidean distances in $[0,1]^d$. Additional experiments (1D visualizations, retraining analysis, acquisition agnosticism, multi-seed robustness) are in Appendix~\ref{app:full_experiments}.

\begin{figure}[t]
\centering
\includegraphics[width=\linewidth]{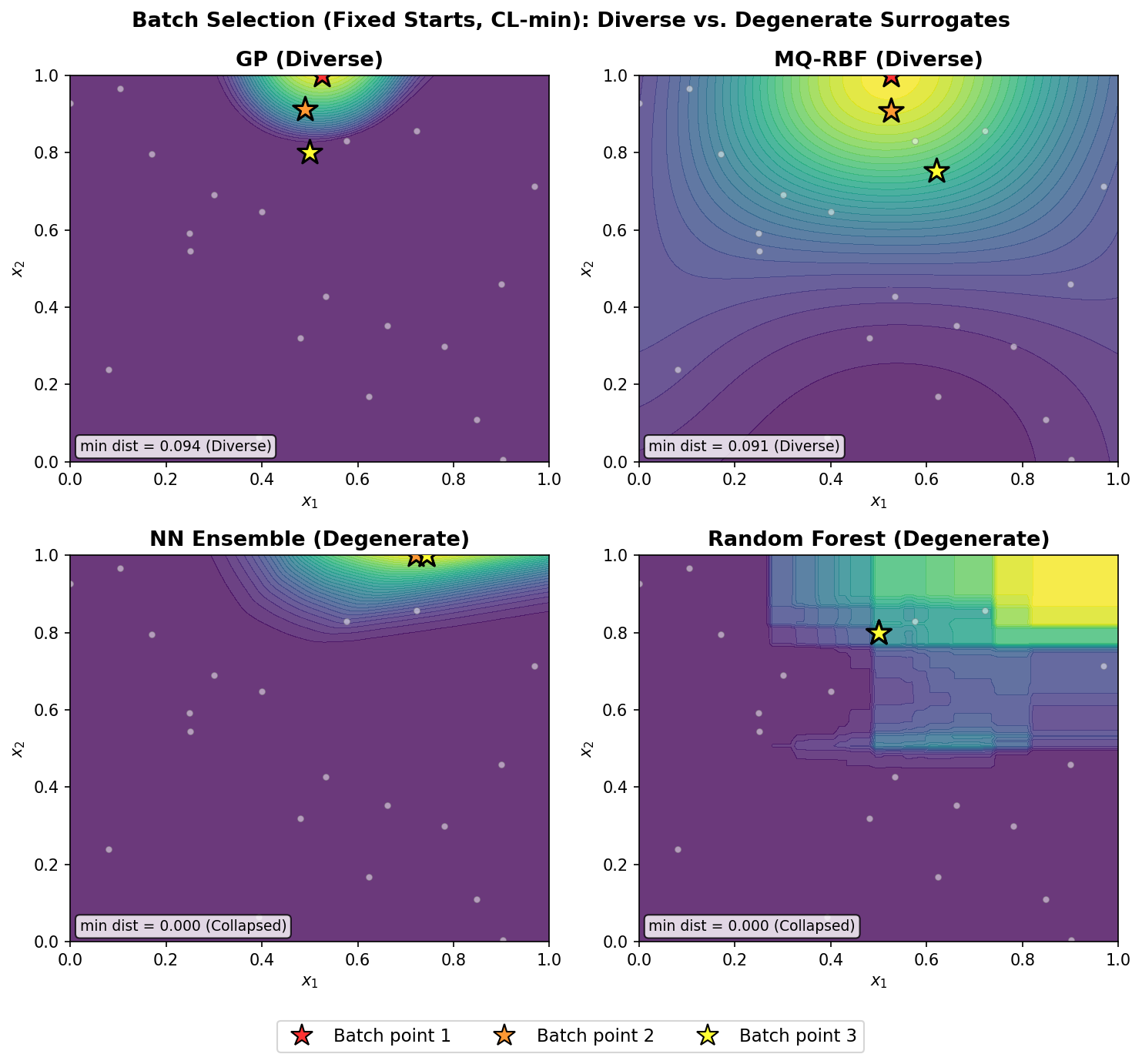}
\caption{Overview: batch selection with CL-min and fixed optimizer starts on a 2D test function ($q{=}3$). GP and MQ-RBF produce diverse batches (top); NN and RF collapse to the same location (bottom).}
\label{fig:overview}
\end{figure}

\subsection{Structural Diversity Diagnostic}
\label{sec:diagnostic}

Corollary~\ref{cor:parametric_degeneracy} predicts that surrogates lacking efficient conditioning produce degenerate batches under a deterministic optimizer. We verify this with the \textbf{Structural Diversity Diagnostic} (SDD): three deterministic L-BFGS-B starting points (placed at $0.2\mathbf{1}_d$, $0.5\mathbf{1}_d$, $0.8\mathbf{1}_d$ along the diagonal of $[0,1]^d$, where $\mathbf{1}_d$ is the all-ones vector) are held identical across batch iterations, so any observed diversity must arise from the model's structural response to pseudo-observations.

Table~\ref{tab:sdd} confirms the prediction on Hartmann-6D ($n_\text{init}{=}30$): GP and MQ-RBF produce diverse batches; NN and RF collapse (distance $= 0$), exactly as the theory requires. Pre-drawn random restarts do not rescue NN/RF (Appendix~\ref{app:full_experiments}), and the classification is perfectly robust across 5 seeds with different Latin Hypercube Sampling (LHS) initializations (Appendix, Table~\ref{tab:multi_seed}).

\begin{table}[t]
\centering
\caption{Structural Diversity Diagnostic on Hartmann-6D.}
\label{tab:sdd}
\small
\begin{tabular}{lcccl}
\toprule
\textbf{Surrogate} & \textbf{$q$} & \textbf{Min Dist} & \textbf{Mean Dist} & \textbf{Diverse?} \\
\midrule
GP                & 2 & 0.300 & 0.300 & Yes \\
MQ-RBF            & 2 & 0.421 & 0.421 & Yes \\
NN Ensemble       & 2 & 0.000 & 0.000 & No \\
Random Forest     & 2 & 0.000 & 0.000 & No \\
\midrule
GP                & 3 & 0.157 & 0.209 & Yes \\
MQ-RBF            & 3 & 0.421 & 0.909 & Yes \\
NN Ensemble       & 3 & 0.000 & 0.000 & No \\
Random Forest     & 3 & 0.000 & 0.000 & No \\
\bottomrule
\end{tabular}
\end{table}

\subsection{Batch BO Convergence}

We evaluate batch BO on Hartmann-6D, Ackley-8D, and Levy-10D, comparing CL-min, CL-max, CL-mean, KB, random batch, and sequential ($q{=}1$) with $q \in \{2, 3\}$, 20 seeds (10 for Levy), budget of 50 evaluations, and $2d$ LHS initialization. We focus on $q \leq 3$, the regime where pseudo-observation strategies are principally deployed \cite{chevalier2013fast, azimi2012hybrid}. Figure~\ref{fig:conv} shows Hartmann-6D convergence: CL-min and KB achieve near-sequential performance at equal evaluation budgets, confirming that pseudo-observation batch selection maintains sample efficiency; at equal \emph{wall-clock time} (same number of BO iterations), batch methods are strictly superior since they evaluate $q$ points per iteration.\footnote{Wilcoxon signed-rank tests (Hartmann-6D, $q{=}3$, 20 seeds): CL-min vs.\ KB $p = 0.37$ (n.s.); CL-min vs.\ Random $p < 0.001$; KB vs.\ Random $p < 0.001$. Similar patterns hold across benchmarks.} CL-max converges more slowly; random batch is worst. Ackley-8D and Levy-10D results (Appendix~\ref{app:full_experiments}) show consistent patterns.

\begin{figure}[t]
\centering
\includegraphics[width=\linewidth]{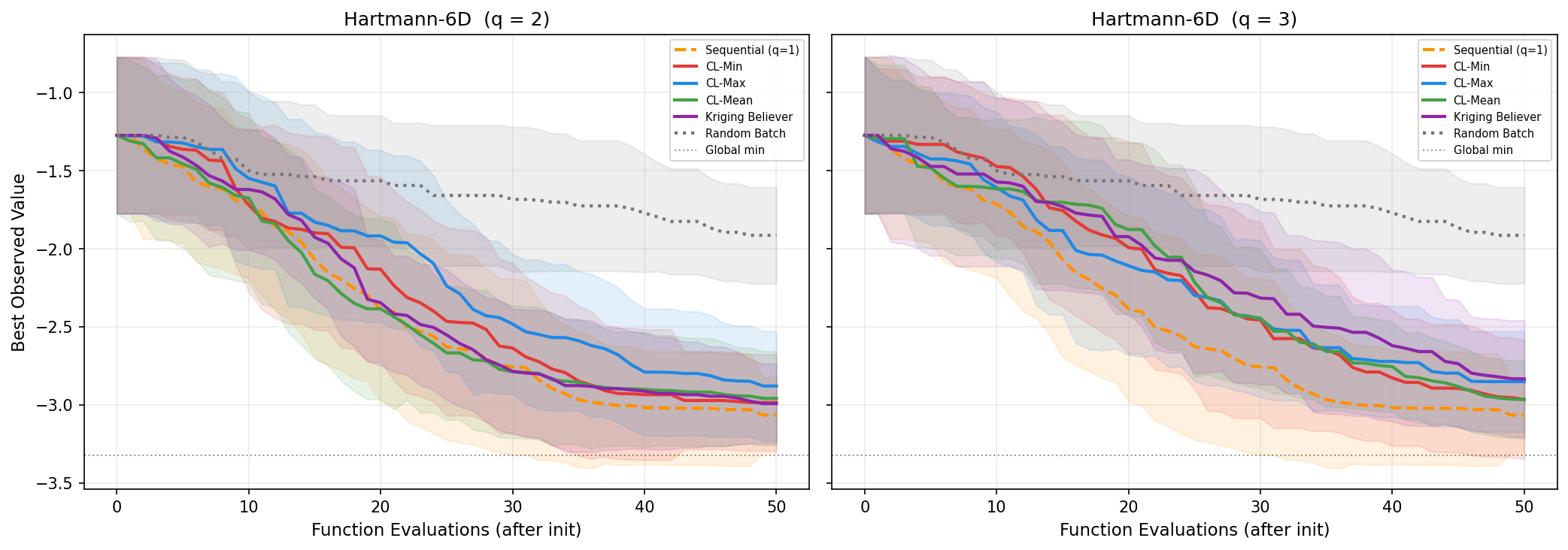}
\caption{Convergence on Hartmann-6D (20 seeds, $\pm$1 std). CL-min and KB achieve near-sequential performance with batch parallelism.}
\label{fig:conv}
\end{figure}

\subsection{Computational Cost}

Table~\ref{tab:timing} shows wall-clock time per batch step ($n{=}50$, $d{=}6$, $q{=}3$). GP conditioning is ${\sim}250\times$ faster than NN retraining; MQ-RBF is ${\sim}3000\times$ faster. End-to-end loop times (Appendix, Table~\ref{tab:loop_timing}) show CL/KB loops complete in under 6 seconds, within the same order of magnitude as BoTorch's q-EI.\footnote{Table~\ref{tab:timing} measures the conditioning/retraining step alone. Per-batch times in Table~\ref{tab:qei} include GP fitting once per BO iteration, acquisition optimization with 10 L-BFGS-B restarts, and $q{-}1$ conditioning steps. Table~\ref{tab:loop_timing} reports full BO loop wall-clock.}

\begin{table}[t]
\centering
\caption{Wall-clock time per batch step. Efficient conditioning is $250$--$3000\times$ faster than retraining.}
\label{tab:timing}
\small
\begin{tabular}{lrr}
\toprule
\textbf{Surrogate} & \textbf{Time/step (s)} & \textbf{vs.\ NN retrain} \\
\midrule
GP (condition)            & $0.0004 \pm 0.0002$ & $253\times$ faster \\
MQ-RBF (re-solve)         & $< 0.0001$           & $3124\times$ faster \\
NN Ensemble (retrain)      & $0.103 \pm 0.028$   & $1\times$ (baseline) \\
Random Forest (rebuild)    & $0.058 \pm 0.003$   & $1.8\times$ faster \\
\bottomrule
\end{tabular}
\end{table}

\subsection{Real-World Benchmark: SVM Hyperparameter Tuning}

We optimize SVM hyperparameters (4D) on the Breast Cancer dataset (569 samples), with $n_\text{init} = 10$, budget $= 30$, $q{=}2$, over 10 seeds. GP + CL-min achieves the lowest mean error ($0.0188 \pm 0.0014$) versus sequential ($0.0193 \pm 0.0008$) and random ($0.0197 \pm 0.0015$); differences against sequential are within one standard deviation, but CL-min provides the additional benefit of $2\times$ wall-clock parallel speedup. Details and convergence figure in Appendix~\ref{app:full_experiments}.

\subsection{Local Penalization Comparison}
\label{sec:lp_experiment}

We compare CL-min, KB, and explicit LP \cite{gonzalez2016batch} (penalty: $\varphi = 1 - \exp(-\|\mathbf{x} - \mathbf{x}_j\|^2 / 2\ell^2)$) on Hartmann-6D and Ackley-8D (20 seeds, $q \in \{2,3\}$). Table~\ref{tab:lp_comparison} shows CL-min and KB achieve the best final values (CL-min leads on Hartmann-6D; KB on Ackley-8D $q{=}2$), with LP slightly trailing; all three significantly outperform random batch. This validates Section~\ref{sec:local_penalization}: CL/KB's implicit penalty is at least as effective as LP's explicit penalty while requiring no Lipschitz constant estimation. Full results in Appendix~\ref{app:full_experiments}.

\begin{table}[t]
\centering
\caption{CL/KB vs.\ Local Penalization: final best (mean $\pm$ std, 20 seeds).}
\label{tab:lp_comparison}
\small
\begin{tabular}{llcc}
\toprule
& & \textbf{Hartmann-6D} & \textbf{Ackley-8D} \\
\textbf{Strategy} & $q$ & Final Best & Final Best \\
\midrule
CL-Min             & 2 & $\mathbf{-3.08 \pm 0.32}$ & $4.34 \pm 0.76$ \\
Kriging Believer   & 2 & $-2.97 \pm 0.25$ & $\mathbf{4.23 \pm 0.84}$ \\
Local Penalization & 2 & $-2.79 \pm 0.28$ & $4.79 \pm 0.82$ \\
Random Batch       & 2 & $-1.91 \pm 0.31$ & $7.11 \pm 0.72$ \\
\midrule
CL-Min             & 3 & $\mathbf{-3.02 \pm 0.30}$ & $\mathbf{4.60 \pm 1.02}$ \\
Kriging Believer   & 3 & $-2.83 \pm 0.41$ & $4.99 \pm 0.93$ \\
Local Penalization & 3 & $-2.60 \pm 0.50$ & $5.20 \pm 0.87$ \\
Random Batch       & 3 & $-1.91 \pm 0.31$ & $7.11 \pm 0.72$ \\
\midrule
Sequential ($q{=}1$)     & 1 & $-2.98 \pm 0.34$ & $4.08 \pm 1.23$ \\
\bottomrule
\end{tabular}
\end{table}

\subsection{$q$-EI Joint Optimization Comparison}
\label{sec:qei_experiment}

We compare CL-min and KB against BoTorch's \texttt{qExpectedImprovement}~\cite{balandat2020botorch} on Hartmann-6D ($q{=}3$, 20 seeds).\footnote{Tables~\ref{tab:lp_comparison} and~\ref{tab:qei} report different absolute values for CL-min and KB because the two experiments use different seed sets and surrogate stacks (LP comparison: scikit-learn GP + custom L-BFGS-B; q-EI comparison: BoTorch \texttt{SingleTaskGP} with the seed protocol shared with q-EI). Differences in the CL-min vs.\ KB ordering between the two tables are within one within-method standard deviation ($0.29$--$0.54$), consistent with the noise-level discussion below.} Table~\ref{tab:qei} shows KB achieves performance indistinguishable from q-EI ($-2.99$ vs.\ $-2.97$), with comparable diversity. Greedy conditioning loses little relative to joint optimization, while remaining simpler and model-agnostic; within-method standard deviations ($0.29$--$0.54$) are comparable to the between-method gaps, so ranking differences across Tables~\ref{tab:lp_comparison} and~\ref{tab:qei} are within noise.

\begin{table}[t]
\centering
\caption{CL/KB vs.\ q-EI (BoTorch) on Hartmann-6D ($q{=}3$, 20 seeds).}
\label{tab:qei}
\small
\begin{tabular}{lccc}
\toprule
\textbf{Strategy} & \textbf{Final Best} & \textbf{Diversity} & \textbf{Time/batch (s)} \\
\midrule
CL-Min (ours)         & $-2.74 \pm 0.54$ & $0.79 \pm 0.35$ & 0.22 \\
Kriging Believer (ours) & $\mathbf{-2.99 \pm 0.29}$ & $0.95 \pm 0.32$ & 0.16 \\
q-EI (BoTorch)        & $-2.97 \pm 0.39$ & $0.95 \pm 0.25$ & 0.09 \\
\bottomrule
\end{tabular}
\end{table}

\subsection{Noisy Observations}
\label{sec:noisy_experiment}

We repeat the Hartmann-6D experiment ($q{=}3$, 20 seeds) with $\sigma_\text{noise} = 0.1 \cdot \text{std}(y_\text{init})$. Table~\ref{tab:noisy} shows that CL-min and KB continue to produce diverse, effective batches. Diversity is largely unchanged (CL-min: $0.74$ noisy vs.\ $0.79$ noise-free in Table~\ref{tab:qei}). Convergence degrades slightly, as expected, but the gap over random batch selection remains large (Wilcoxon $p < 0.001$).

\begin{table}[t]
\centering
\caption{Noisy Hartmann-6D ($\sigma_\text{noise}{=}0.1 \cdot \text{std}(y)$, $q{=}3$, 20 seeds).}
\label{tab:noisy}
\small
\begin{tabular}{lcc}
\toprule
\textbf{Strategy} & \textbf{Final Best} & \textbf{Diversity} \\
\midrule
CL-Min             & $\mathbf{-2.95 \pm 0.29}$ & $0.74 \pm 0.12$ \\
Kriging Believer   & $-2.90 \pm 0.50$ & $0.95 \pm 0.07$ \\
Random Batch       & $-1.88 \pm 0.35$ & $0.97 \pm 0.04$ \\
\bottomrule
\end{tabular}
\end{table}

\section{Related Work}
\label{sec:related}

Ginsbourger et al.\ introduced CL and KB~\cite{ginsbourger2010kriging, ginsbourger2008multi}; subsequent work focused on regret guarantees~\cite{contal2013parallel, desautels2014parallelizing}, asynchronous extensions~\cite{alvi2019asynchronous}, and high-dimensional settings~\cite{wang2018batched}, treating CL/KB as black-box procedures without characterizing surrogate requirements. Alternative batch methods include LP~\cite{gonzalez2016batch} (explicit repulsion), $q$-EI~\cite{chevalier2013fast, wilson2018maximizing} (joint optimization), DPPs~\cite{kathuria2016batched} (repulsive kernels), Thompson sampling~\cite{kandasamy2018parallelised} (posterior sampling), SOBER~\cite{adachi2023sober} (hallucination-free), and subspace-based scalable batch BO~\cite{zhan2024essi}; our analysis explains \emph{why} the greedy CL/KB decomposition works and is complementary to these methods. On surrogates, GPs dominate BO~\cite{rasmussen2006gaussian, shahriari2015taking}, with alternatives including random forests (SMAC~\cite{hutter2011sequential}), deep NNs (DNGO~\cite{snoek2015scalable}), and sparse GPs~\cite{hensman2013gaussian}; BoTorch's~\cite{balandat2020botorch} ``fantasy models'' implicitly rely on the efficient conditioning property we formalize, while prior comparisons~\cite{eggensperger2013towards} focus on predictive accuracy rather than batch diversity.

\section{Limitations}
\label{sec:limitations}

\begin{itemize}
\item \textbf{Fixed hyperparameters.} Our analysis assumes kernel hyperparameters remain fixed during batch construction; experiments confirm effectiveness with per-batch optimization, but interaction with hyperparameter adaptation deserves further study.
\item \textbf{MQ-RBF uncertainty.} The MQ-RBF uncertainty estimate uses a power-function heuristic (Appendix~\ref{app:mqrbf_uncertainty}) rather than a probabilistic model, which suffices for demonstrating efficient conditioning but lacks GP-level calibration.
\item \textbf{Batch size regime.} The main experiments focus on $q \in \{2, 3\}$, consistent with greedy methods yielding diminishing returns beyond $q \approx 5$~\cite{azimi2012hybrid} and exact $q$-EI becoming intractable for $q \geq 10$~\cite{chevalier2013fast}. An additional $q{=}10$ experiment (Appendix~\ref{app:large_batch}) confirms the diversity mechanism scales, though convergence efficiency degrades for later batch members as predicted by Proposition~\ref{prop:suppression_radius}.
\item \textbf{Benchmark scope.} We cover dimensions 6--10 and a real-world task; extension to $d > 20$ or very high noise remains open (Section~\ref{sec:noisy_experiment} confirms robustness under moderate noise).
\end{itemize}

\section{Conclusion}
\label{sec:conclusion}

We identified \emph{efficient conditioning} as the key surrogate property for pseudo-observation batch selection. Batch points are provably distinct with separation of order $\ell$ (Propositions~\ref{prop:no_duplicate}--\ref{prop:suppression_radius}) for any $\sigma$-monotone acquisition function (Proposition~\ref{prop:acq_general}), unifying CL, KB, and fantasy models with bridges to LP and DPPs. Efficient conditioning extends to MQ-RBF networks, while parametric surrogates remain degenerate even when retrained (RFs) or at $15\times$ GP cost (NNs). The SDD provides a reusable diagnostic robust across seeds, acquisition functions, and noise; CL/KB matches or outperforms LP across benchmarks and is competitive with joint $q$-EI within noise. These findings establish efficient conditioning as a fundamental organizing principle for batch Bayesian Optimization. All hyperparameters, protocols, and compute resources are detailed in Appendices~\ref{app:hyperparams}--\ref{app:compute}.

\bibliographystyle{plainnat}
\bibliography{references}

\appendix
\section{Extended Background}
\label{app:extended_background}

\subsection{Gaussian Process Regression}

A Gaussian Process defines a distribution over functions $f: \mathcal{X} \rightarrow \mathbb{R}$, specified by a mean function $m(\mathbf{x})$ and covariance kernel $k(\mathbf{x}, \mathbf{x}')$ \cite{rasmussen2006gaussian}:
\begin{equation}
f(\mathbf{x}) \sim \mathcal{GP}(m(\mathbf{x}), k(\mathbf{x}, \mathbf{x}')).
\end{equation}

Given dataset $\mathcal{D} = \{(\mathbf{x}_i, y_i)\}_{i=1}^n$, define:
\begin{align}
\mathbf{X} &= [\mathbf{x}_1, \ldots, \mathbf{x}_n]^\top, \\
\mathbf{y} &= [y_1, \ldots, y_n]^\top, \\
\mathbf{m} &= [m(\mathbf{x}_1), \ldots, m(\mathbf{x}_n)]^\top.
\end{align}

Let the kernel matrix $\mathbf{K} \in \mathbb{R}^{n \times n}$ with $K_{ij} = k(\mathbf{x}_i, \mathbf{x}_j)$, and the regularized system $\mathbf{K}_y = \mathbf{K} + \sigma_n^2 \mathbf{I}$ where $\sigma_n^2$ is the observation noise variance. For test point $\mathbf{x}_*$, the covariance vector is $\mathbf{k}_* = [k(\mathbf{x}_*, \mathbf{x}_1), \ldots, k(\mathbf{x}_*, \mathbf{x}_n)]^\top$.

The posterior predictive distribution is Gaussian:
\begin{align}
\mu(\mathbf{x}_*) &= m(\mathbf{x}_*) + \mathbf{k}_*^\top \boldsymbol{\alpha}, \\
\sigma^2(\mathbf{x}_*) &= k(\mathbf{x}_*, \mathbf{x}_*) - \mathbf{k}_*^\top \mathbf{K}_y^{-1} \mathbf{k}_*,
\end{align}
where $\boldsymbol{\alpha}$ solves $\mathbf{K}_y \boldsymbol{\alpha} = \mathbf{y} - \mathbf{m}$.

The predictive mean is a kernel-weighted linear combination of observed targets, while the variance quantifies uncertainty based on proximity to observed data. In the noise-free limit ($\sigma_n^2 \rightarrow 0$), the GP interpolates the data.

\subsection{Acquisition Functions}

BO selects evaluation points by maximizing $\alpha: \mathcal{X} \rightarrow \mathbb{R}$, combining $\mu(\mathbf{x})$ and $\sigma(\mathbf{x})$ to balance exploration and exploitation \cite{shahriari2015taking}.

Expected Improvement (EI) \cite{jones1998efficient} measures expected improvement over $f^+ = \min_i y_i$:
\begin{equation}
\text{EI}(\mathbf{x}) = (f^+ - \mu(\mathbf{x})) \Phi(Z) + \sigma(\mathbf{x}) \phi(Z),
\end{equation}
where $Z = (f^+ - \mu(\mathbf{x})) / \sigma(\mathbf{x})$.

Upper Confidence Bound (UCB) \cite{srinivas2010gaussian}: in the original maximization formulation, $\text{UCB}(\mathbf{x}) = \mu(\mathbf{x}) + \sqrt{\beta}\,\sigma(\mathbf{x})$. For minimization (our setting), Algorithm~\ref{alg:batch} maximizes the analogous acquisition $\alpha_{\text{UCB}}(\mathbf{x}) = -\mu(\mathbf{x}) + \sqrt{\beta}\,\sigma(\mathbf{x})$, equivalent to minimizing the lower confidence bound $\text{LCB}(\mathbf{x}) = \mu(\mathbf{x}) - \sqrt{\beta}\,\sigma(\mathbf{x})$. We refer to this acquisition as ``UCB'' throughout for brevity. Both the maximization and minimization forms are non-decreasing in $\sigma$ for fixed $\mu$, satisfying the premise of Proposition~\ref{prop:acq_general}.

EI and UCB both depend explicitly on $\mu(\mathbf{x})$ and $\sigma(\mathbf{x})$; modifications to $\mathcal{D}$ that alter these quantities reshape the acquisition landscape.

\section{Extended Proofs}
\label{app:proofs}

\subsection{Proposition~\ref{prop:no_duplicate}: No-Duplicate Guarantee}

\begin{proof}
We prove by induction on batch index $j$. Let $\mathcal{D}_j = \mathcal{D} \cup \{(\mathbf{x}_i, \tilde{y}_i)\}_{i=1}^{j-1}$ and $\mathcal{D}_{j+1} = \mathcal{D}_j \cup \{(\mathbf{x}_j, \tilde{y}_j)\}$.

\textbf{Step 1: Variance reduction.} From Eq.~\ref{eq:var_update}:
\begin{equation}
\sigma^2_{\mathcal{D}_{j+1}}(\mathbf{x}_j) = \frac{\sigma^2_{\mathcal{D}_j}(\mathbf{x}_j) \cdot \sigma_n^2}{\sigma^2_{\mathcal{D}_j}(\mathbf{x}_j) + \sigma_n^2} < \sigma^2_{\mathcal{D}_j}(\mathbf{x}_j).
\end{equation}

\textbf{Step 2: Acquisition reduction.} For any $\alpha$ non-decreasing in $\sigma$, the strict decrease in $\sigma_{\mathcal{D}_{j+1}}(\mathbf{x}_j)$ implies $\alpha(\mathbf{x}_j \mid \mathcal{D}_{j+1}) < \alpha(\mathbf{x}_j \mid \mathcal{D}_j)$.

\textbf{Step 3: Continuity.} By continuity of $k$, suppression extends to an open neighborhood $B_\delta(\mathbf{x}_j)$. Compactness of $\mathcal{X}$ ensures every maximizer of $\alpha(\cdot \mid \mathcal{D}_{j+1})$ lies outside $B_\delta(\mathbf{x}_j)$. When unique, $\mathbf{x}_{j+1}$ is uniquely determined; otherwise the conclusion holds for every maximizer.
\end{proof}

\subsection{Proposition~\ref{prop:suppression_radius}: Variance Suppression Radius}

\begin{proof}
The exact rank-one variance update uses the posterior covariance $k_{\mathcal{D}}(\mathbf{x}, \mathbf{x}_*)$ in the numerator of Eq.~\ref{eq:var_update}. Under the prior-dominated assumption ($\sigma^2_{\mathcal{D}}(\mathbf{x}_*) \approx \sigma_f^2$, i.e., $\mathbf{x}_*$ is sufficiently far from training data), $k_{\mathcal{D}}(\mathbf{x}, \mathbf{x}_*) \approx k(\mathbf{x}, \mathbf{x}_*)$, so for the SE kernel $k(\mathbf{x}, \mathbf{x}_*)^2 = \sigma_f^4 \exp(-\|\mathbf{x} - \mathbf{x}_*\|^2/\ell^2)$. The variance reduction is:
\begin{equation}
\Delta \sigma^2(\mathbf{x}) = \frac{\sigma_f^4 \exp(-\|\mathbf{x} - \mathbf{x}_*\|^2/\ell^2)}{\sigma^2_{\mathcal{D}}(\mathbf{x}_*) + \sigma_n^2}.
\end{equation}
The fractional reduction:
\begin{equation}
\frac{\Delta\sigma^2(\mathbf{x})}{\sigma^2_{\mathcal{D}}(\mathbf{x})} = \frac{\sigma_f^4 \exp(-\|\mathbf{x} - \mathbf{x}_*\|^2/\ell^2)}{\sigma^2_{\mathcal{D}}(\mathbf{x})(\sigma^2_{\mathcal{D}}(\mathbf{x}_*) + \sigma_n^2)}.
\end{equation}
Setting this $\geq \tau$ and using $\sigma^2_{\mathcal{D}}(\mathbf{x}) \leq \sigma^2_{\mathcal{D}}(\mathbf{x}_*)$ (since $\mathbf{x}_*$ is a local variance maximizer; this substitution gives a conservative lower bound on $r_\tau$):
\begin{equation}
r_\tau \geq \ell\sqrt{\log\!\left(\frac{\sigma_f^4}{\tau \cdot \sigma^2_{\mathcal{D}}(\mathbf{x}_*)(\sigma^2_{\mathcal{D}}(\mathbf{x}_*) + \sigma_n^2)}\right)}.
\end{equation}
For low noise and $\sigma^2_\mathcal{D}(\mathbf{x}_*) \approx \sigma_f^2$: $r_{0.5} \geq 0.83\ell$, confirming $\Omega(\ell)$ separation.
\end{proof}

\subsection{Proposition~\ref{prop:acq_general}: Acquisition Function Generality}

\begin{proof}
The only property of EI used in Proposition~\ref{prop:no_duplicate} is monotonicity in $\sigma$: the strict variance reduction at $\mathbf{x}_1$ (Eq.~\ref{eq:var_update}) implies $\alpha(\mu_{\mathcal{D}'}(\mathbf{x}_1), \sigma_{\mathcal{D}'}(\mathbf{x}_1)) < \alpha(\mu_{\mathcal{D}}(\mathbf{x}_1), \sigma_{\mathcal{D}}(\mathbf{x}_1))$ for any $\alpha$ non-decreasing in $\sigma$. The continuity and compactness arguments follow identically. Proposition~\ref{prop:suppression_radius} depends only on the kernel.
\end{proof}

\subsection{Connection to Local Penalization: Derivation Details}

Under KB, the mean is invariant (Lemma~\ref{prop:kb_invariance}), so $\psi$ depends only on the variance ratio:
\begin{align}
\psi(\mathbf{x};\mathbf{x}_*) &= \frac{\text{EI}(\mu, \sigma')}{\text{EI}(\mu, \sigma)} \approx \frac{\sigma'}{\sigma} = \sqrt{1 - \frac{\delta\sigma^2}{\sigma^2(\mathbf{x})}},
\end{align}
where $\delta\sigma^2 = k(\mathbf{x},\mathbf{x}_*)^2/(\sigma^2(\mathbf{x}_*) + \sigma_n^2)$, giving the form in Eq.~\ref{eq:implicit_form}. The approximation $\text{EI}'/\text{EI} \approx \sigma'/\sigma$ is tight in the exploration-dominated regime ($\sigma\phi(Z) \gg (f^+ - \mu)\Phi(Z)$); in the exploitation regime the ratio is closer to $1$. Unlike LP, the implicit penalty adapts automatically to the posterior through the kernel and current posterior variance.

\section{Additional Surrogate Analysis}
\label{app:surrogate_analysis}

\subsection{Support Vector Machines}

SVMs learn a decision function $f(\mathbf{x}) = \sum_{i \in SV} \alpha_i k(\mathbf{x}, \mathbf{x}_i) + b$ by solving a quadratic program. Adding a new observation requires re-solving the QP, which is not a closed-form update. SVMs also lack native uncertainty estimates required for acquisition functions.

\subsection{Sparse and Scalable Gaussian Processes}

Inducing-point methods (SVGP, FITC, VFE \cite{hensman2013gaussian}) approximate the full GP with $m \ll n$ pseudo-inputs. With \emph{fixed} inducing points, adding data can be performed efficiently through the low-rank structure, but the fixed basis may not capture local structure effectively. With \emph{updated} inducing points, re-optimization violates efficient conditioning. Sparse GPs are partially compatible; empirical investigation of SVGP-based CL/KB remains open.

\subsection{Lie Strategy Variants}

For models satisfying efficient conditioning, the choice of $\tilde{y}$ determines the acquisition landscape modification:
\begin{itemize}
    \item \textbf{CL-min:} $\tilde{y} = \min_i y_i$. The mean decreases locally, reinforcing the variance reduction and producing strong suppression that encourages exploration.
    \item \textbf{CL-max:} $\tilde{y} = \max_i y_i$. The mean increases locally, partially offsetting variance reduction and creating conflicting signals.
    \item \textbf{CL-mean:} $\tilde{y} = \bar{y}$. Moderate mean shift with balanced behavior.
    \item \textbf{KB:} $\tilde{y} = \mu(\mathbf{x}_j)$. Zero mean shift (Lemma~\ref{prop:kb_invariance}), resulting in pure variance-driven diversification.
\end{itemize}

\subsection{Quantitative EI Suppression}

Using Eqs.~\ref{eq:mean_update}--\ref{eq:var_update}, define $\delta\mu(\mathbf{x}) = w(\mathbf{x}, \mathbf{x}_*)(\tilde{y} - \mu(\mathbf{x}_*))$ and $\delta\sigma^2(\mathbf{x}) = k(\mathbf{x}, \mathbf{x}_*)^2 / (\sigma^2(\mathbf{x}_*) + \sigma_n^2)$. For CL-min, the first-order EI change is:
\begin{equation}
\Delta\text{EI}(\mathbf{x}) \approx -\delta\mu(\mathbf{x}) \Phi(Z) - \frac{\delta\sigma^2(\mathbf{x})}{2\sigma(\mathbf{x})}\phi(Z),
\end{equation}
where $Z = (f^+ - \mu(\mathbf{x}))/\sigma(\mathbf{x})$. The suppression decays as $\exp(-\|\mathbf{x} - \mathbf{x}_*\|^2/\ell^2)$, confirming the lengthscale governs the repulsion radius.

\subsection{Connection to Determinantal Point Processes}
\label{app:dpp_connection}

DPPs \cite{kathuria2016batched} define repulsive distributions with $\Pr(S) \propto \det(\mathbf{L}_S)$. Under CL/KB, the effective batch acquisition can be written as $\alpha_{\text{batch}} = \alpha(\mathbf{x}_1) \prod_{j=2}^{q} \psi(\mathbf{x}_j; \mathbf{x}_{<j})$. Under KB with the SE kernel, $\log \psi \approx -\frac{1}{2} k(\mathbf{x}_j, \mathbf{x}_i)^2 / [\sigma^2(\mathbf{x}_i)(\sigma^2(\mathbf{x}_i) + \sigma_n^2)]$, penalizing kernel-correlated pairs in a manner qualitatively similar to DPP log-determinant repulsion. Both scale with kernel correlation and vanish for well-separated points, but operate through different mechanisms: CL/KB's repulsion emerges from sequential conditioning, while DPPs optimize a combinatorial objective. We do not claim exact L-ensemble equivalence; the functional forms differ (pairwise sum vs.\ log-determinant with higher-order interactions).

\subsection{Efficient Conditioning vs.\ Retraining}
\label{app:conditioning_vs_retraining}

GP conditioning produces deterministic, smooth, spatially localized updates governed by the kernel (Eq.~\ref{eq:var_update}). Retraining a neural network, by contrast, is: (1) \emph{non-local}, as gradient updates affect all weights globally; (2) \emph{stochastic}, since different initializations yield different landscapes; and (3) \emph{non-monotone in distance}, with no guarantee that prediction change decreases with $\|\mathbf{x} - \mathbf{x}_*\|$. Even retrained NNs lack the structured suppression profile of GP conditioning (Section~\ref{sec:requirements}). For random forests, bootstrap sampling dilutes one pseudo-observation across 200 trees, often producing insufficient prediction change. Experiment~7a (Table~\ref{tab:retrained}) validates both predictions.

\section{MQ-RBF Uncertainty: Power-Function Analog}
\label{app:mqrbf_uncertainty}

An MQ-RBF interpolant with data-dependent centers yields $s(\mathbf{x}) = \sum_{i=1}^{n} w_i \phi(\|\mathbf{x} - \mathbf{x}_i\|)$ where $\phi(r) = \sqrt{1 + r^2/\epsilon^2}$. For conditionally positive definite basis functions, the \emph{power function} \cite{wendland2005scattered, iske2003approximation} measures worst-case interpolation error:
\begin{equation}
P_{\mathcal{X}}(\mathbf{x})^2 = \phi(0) - \boldsymbol{\phi}(\mathbf{x})^\top \mathbf{\Phi}^{-1} \boldsymbol{\phi}(\mathbf{x}),
\end{equation}
Note the structural parallel to GP posterior variance: $\sigma^2_{\text{GP}} = k(\mathbf{x}, \mathbf{x}) - \mathbf{k}^\top \mathbf{K}^{-1} \mathbf{k}$. Our uncertainty estimate $\hat{\sigma}(\mathbf{x}) = |1 - \boldsymbol{\phi}(\mathbf{x})^\top \mathbf{\Phi}^{-1} \boldsymbol{\phi}(\mathbf{x})|^{1/2}$ is precisely the power function (since $\phi(0) = 1$ for the normalized Multiquadric). When a pseudo-observation at $\mathbf{x}_*$ is added, the power function drops to zero there and decreases locally, producing the same qualitative suppression geometry as GP conditioning. This is why MQ-RBF passes the SDD. Unlike GP posterior variance, the power function is a deterministic worst-case bound rather than a probabilistic calibration.

\section{Full Experimental Results}
\label{app:full_experiments}

\subsection{Experiment 1: Prediction Changes with Fake Data}

A GP with RBF kernel is fit to a small 1D training set, then a pseudo-observation is inserted without modifying hyperparameters. Figure~\ref{fig:exp1} shows: (i) the mean shifts globally near the inserted point; (ii) the variance contracts locally, suppressing EI.

\begin{figure}[h]
\centering
\includegraphics[width=\linewidth]{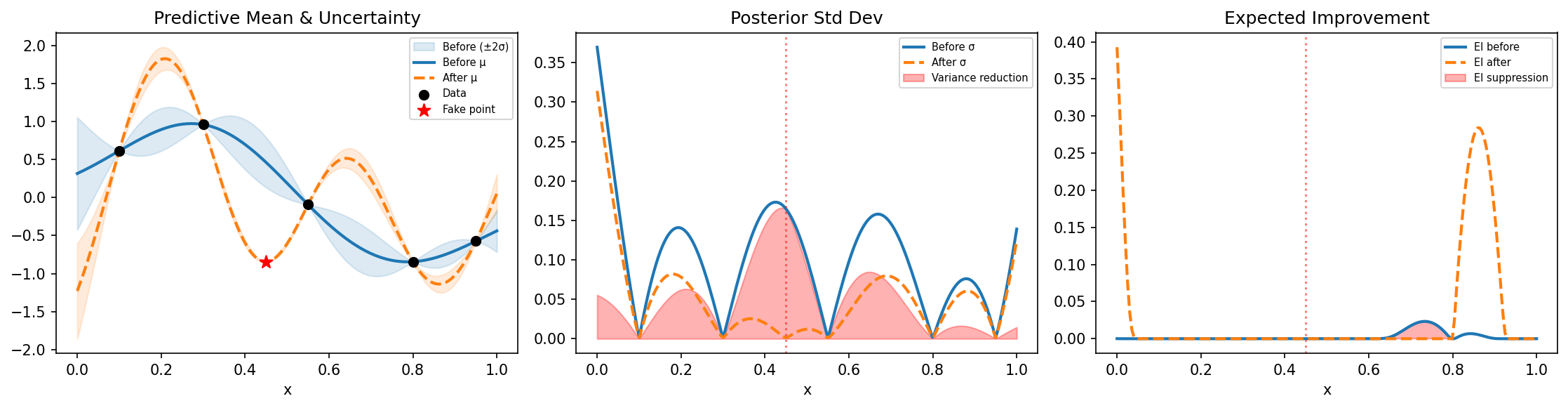}
\caption{Effect of adding a fake observation on GP predictions and EI.}
\label{fig:exp1}
\end{figure}

\subsection{Experiment 2: Surrogate Model Comparison}

Four surrogates (GP, MQ-RBF, NN, RF) on the same 1D problem. After inserting a pseudo-observation, GP and MQ-RBF exhibit measurable prediction changes; NN and RF show no change without retraining (Figure~\ref{fig:exp2}).

\begin{figure}[h]
\centering
\includegraphics[width=\linewidth]{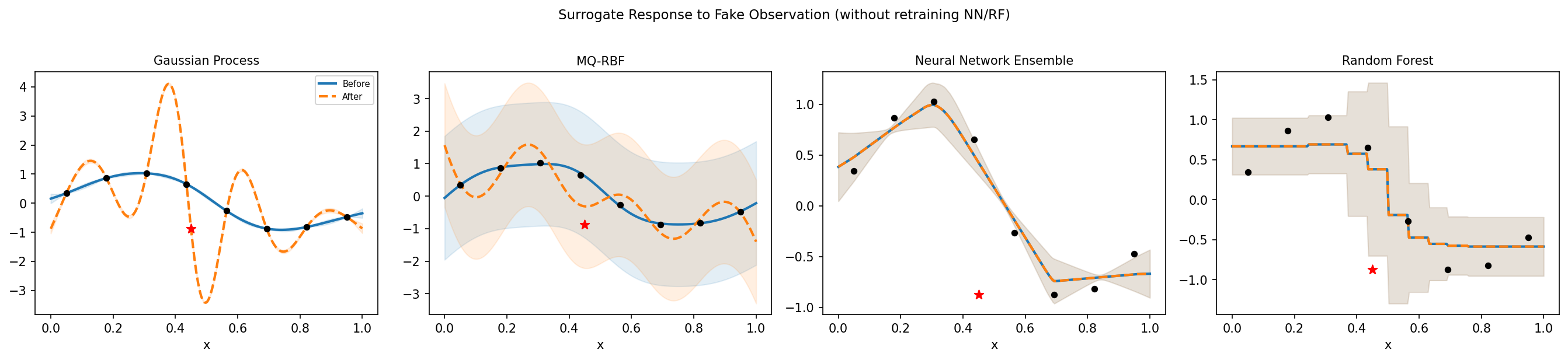}
\caption{Surrogate model response to fake data. GP and MQ-RBF update predictions; NN and RF remain unchanged.}
\label{fig:exp2}
\end{figure}

\subsection{SDD Part B: Fixed vs.\ Random Restarts}

Figure~\ref{fig:exp3a} shows pairwise batch distances under fixed starts; Figure~\ref{fig:exp3b} compares fixed vs.\ pre-drawn random restarts, confirming that diversity is structural.

\begin{figure}[h]
\centering
\includegraphics[width=0.85\linewidth]{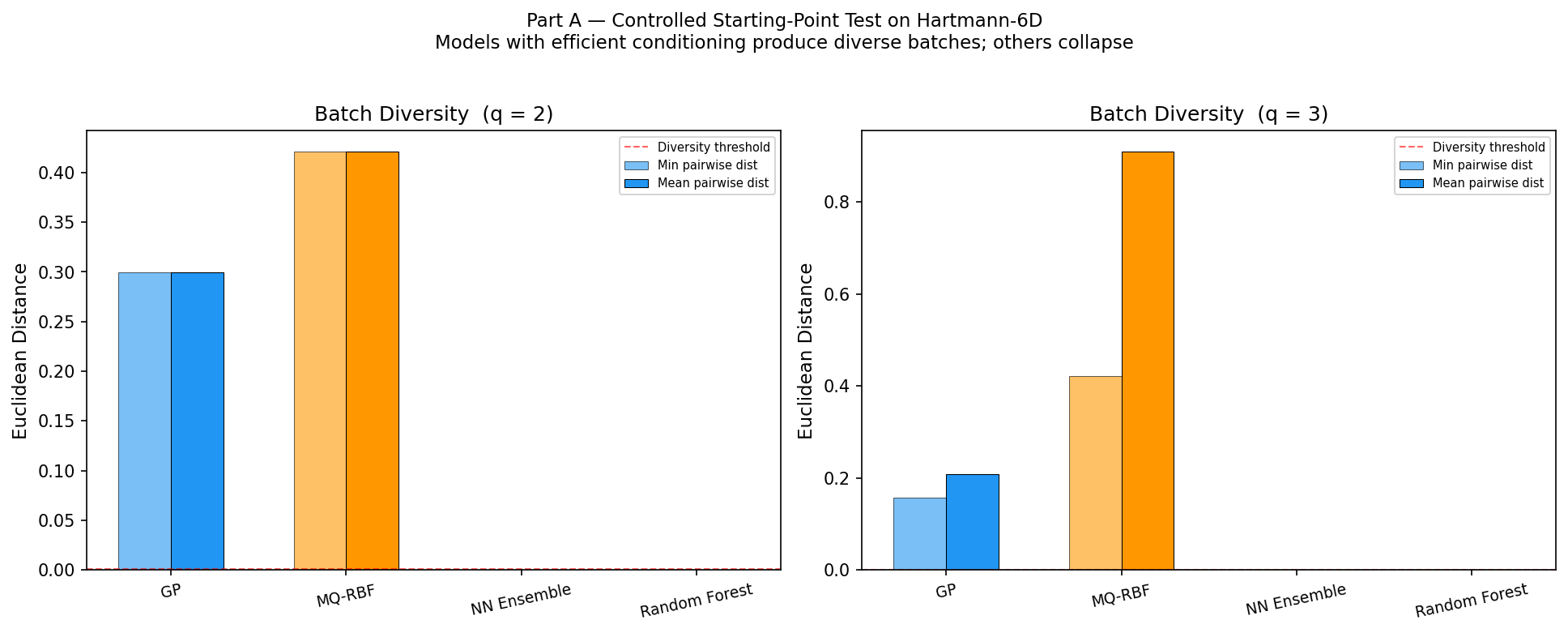}
\caption{SDD Part A: pairwise distances within batches. GP and MQ-RBF maintain diversity; NN and RF collapse.}
\label{fig:exp3a}
\end{figure}

\begin{figure}[h]
\centering
\includegraphics[width=0.85\linewidth]{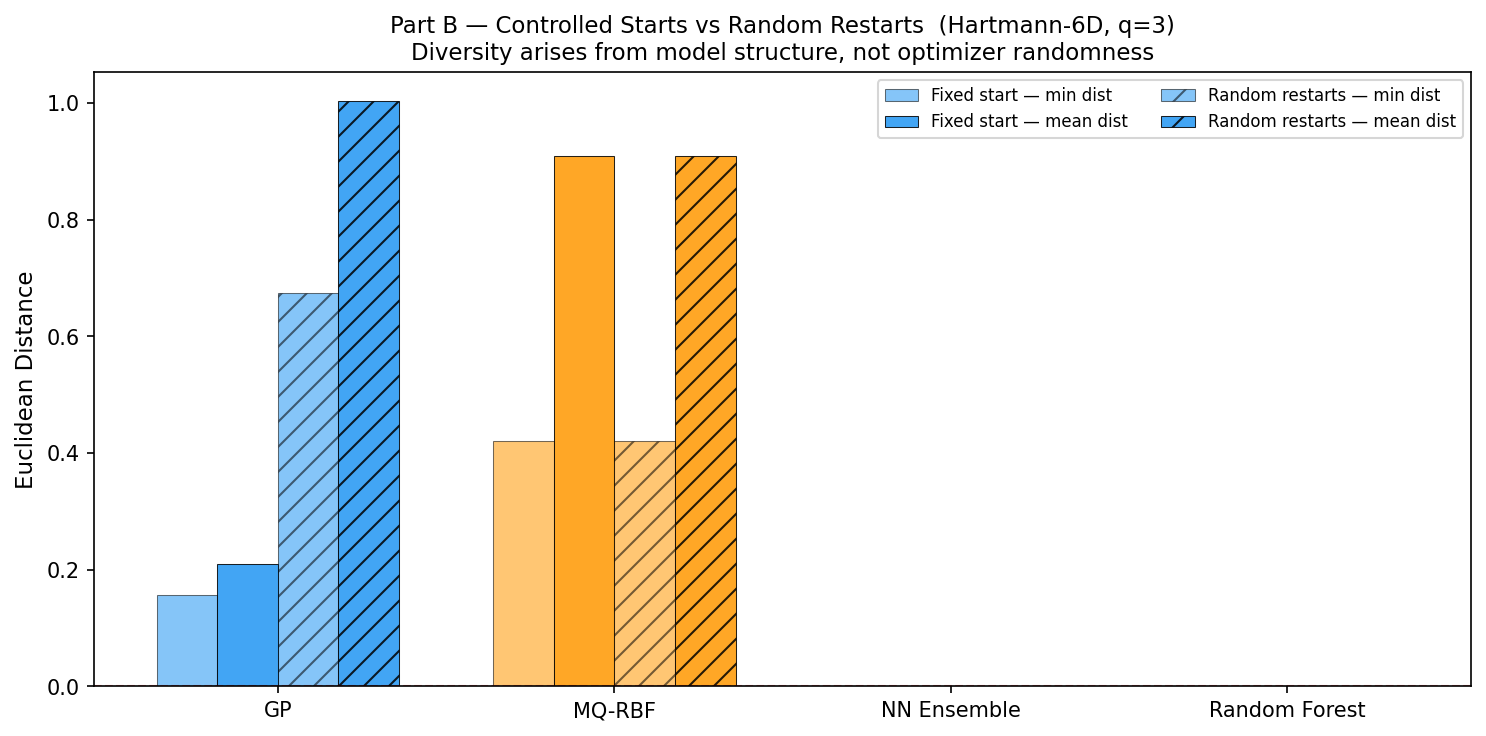}
\caption{SDD Part B: Fixed vs.\ random restarts (reused). GP/MQ-RBF diverse under both; NN/RF degenerate under both.}
\label{fig:exp3b}
\end{figure}

\subsection{Experiment 7a: Efficient Conditioning vs.\ Retraining}

\begin{table}[h]
\centering
\caption{SDD with retrained parametric models (Hartmann-6D, $q{=}3$).}
\label{tab:retrained}
\small
\begin{tabular}{lcccc}
\toprule
\textbf{Surrogate} & \textbf{Min Dist} & \textbf{Mean Dist} & \textbf{Diverse?} & \textbf{Time (s)} \\
\midrule
GP (condition)         & 0.157 & 0.209 & Yes & 0.04 \\
MQ-RBF (re-solve)      & 0.421 & 0.909 & Yes & 0.09 \\
NN (no retrain)        & 0.000 & 0.000 & No  & 0.02 \\
NN (retrained)         & 0.378 & 0.477 & Yes & 0.60 \\
RF (no rebuild)        & 0.000 & 0.000 & No  & 0.26 \\
RF (rebuilt)           & 0.000 & 0.000 & No  & 0.39 \\
\bottomrule
\end{tabular}
\end{table}

Retrained NNs produce diversity at $15\times$ the cost of GP conditioning. Rebuilt random forests remain degenerate: bootstrap sampling dilutes the pseudo-observation across 200 trees.

\subsection{Experiment 7b: Acquisition Function Agnosticism}

\begin{table}[h]
\centering
\caption{SDD with EI vs.\ UCB (Hartmann-6D, $q{=}3$).}
\label{tab:acq_agnostic}
\small
\begin{tabular}{lcccc}
\toprule
 & \multicolumn{2}{c}{\textbf{EI}} & \multicolumn{2}{c}{\textbf{UCB ($\beta{=}2$)}} \\
\cmidrule(lr){2-3} \cmidrule(lr){4-5}
\textbf{Surrogate} & Min Dist & Mean Dist & Min Dist & Mean Dist \\
\midrule
GP            & 0.157 & 0.209 & 0.169 & 0.206 \\
MQ-RBF        & 0.421 & 0.909 & 0.422 & 0.907 \\
NN Ensemble   & 0.000 & 0.000 & 0.000 & 0.000 \\
Random Forest & 0.000 & 0.000 & 0.000 & 0.000 \\
\bottomrule
\end{tabular}
\end{table}

Diversity patterns are identical under EI and UCB, validating Proposition~\ref{prop:acq_general}.

\subsection{Experiment 8: SDD Robustness Across Initial Datasets}

\begin{table}[h]
\centering
\caption{SDD robustness across 5 initial datasets (Hartmann-6D, $q{=}3$).}
\label{tab:multi_seed}
\small
\begin{tabular}{lccc}
\toprule
\textbf{Surrogate} & \textbf{Seeds Diverse} & \textbf{Min Dist (mean $\pm$ std)} & \textbf{Consistent?} \\
\midrule
GP            & 5/5 & $0.480 \pm 0.314$ & Yes \\
MQ-RBF        & 5/5 & $0.287 \pm 0.073$ & Yes \\
NN Ensemble   & 0/5 & $0.000 \pm 0.000$ & Yes \\
Random Forest & 0/5 & $0.000 \pm 0.000$ & Yes \\
\bottomrule
\end{tabular}
\end{table}

The diverse/degenerate classification is perfectly consistent across seeds.

\subsection{Ackley-8D Convergence and Diversity}

Figures~\ref{fig:conv_ackley}--\ref{fig:div_ackley} complement the Hartmann-6D main-text results. The same patterns hold: CL-min and KB achieve near-sequential convergence with consistently high batch diversity.

\begin{figure}[h]
\centering
\includegraphics[width=\linewidth]{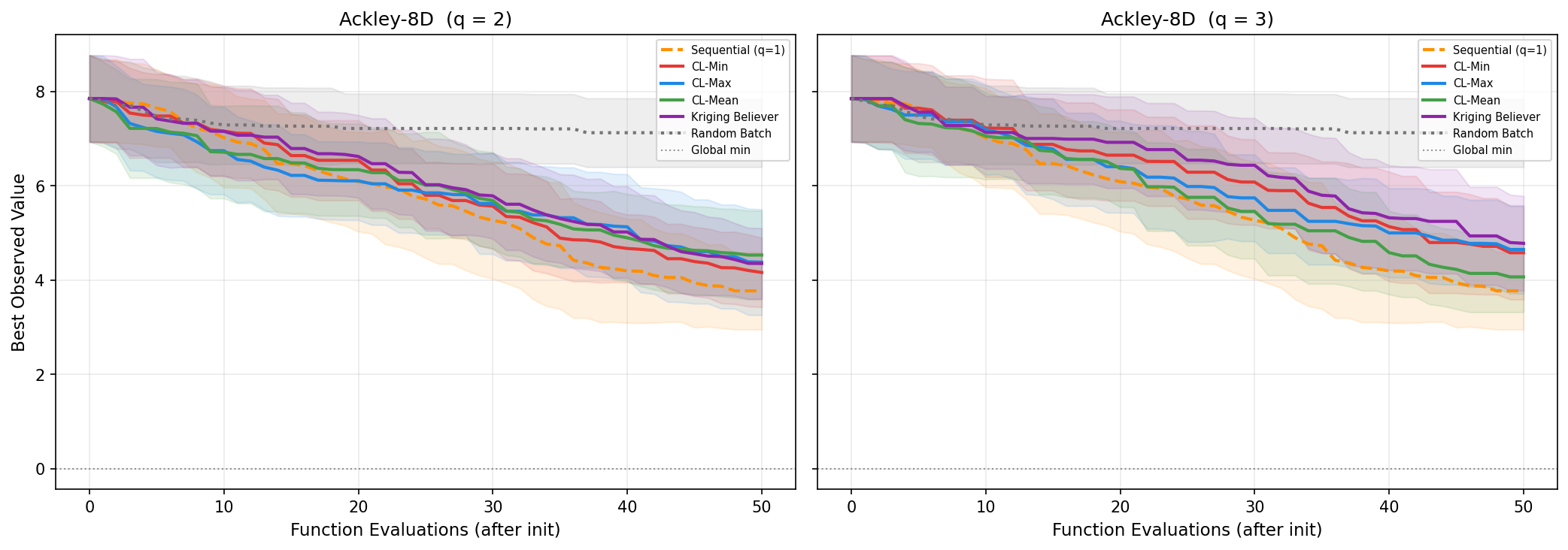}
\caption{Convergence on Ackley-8D (20 seeds, $\pm$1 std).}
\label{fig:conv_ackley}
\end{figure}

\begin{figure}[h]
\centering
\includegraphics[width=\linewidth]{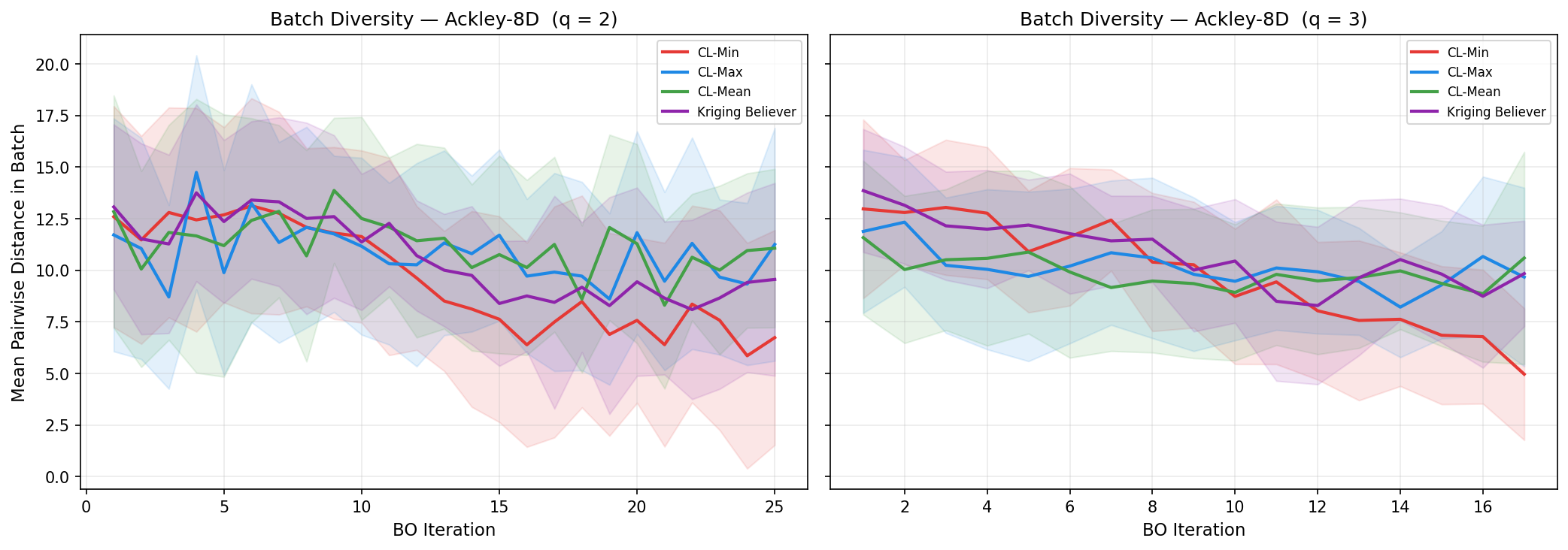}
\caption{Batch diversity on Ackley-8D.}
\label{fig:div_ackley}
\end{figure}

\subsection{Hartmann-6D Diversity}

Figure~\ref{fig:diversity} tracks batch diversity across BO iterations on Hartmann-6D, showing CL-min maintains the highest diversity throughout.

\begin{figure}[h]
\centering
\includegraphics[width=\linewidth]{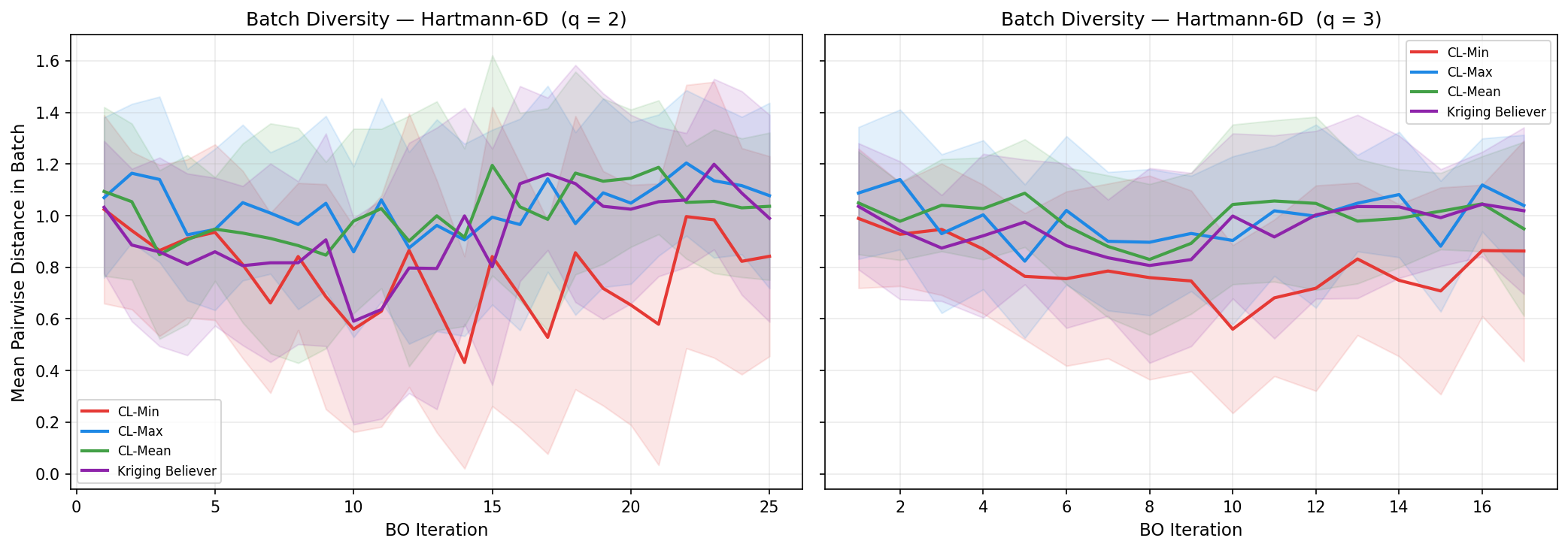}
\caption{Batch diversity across BO iterations on Hartmann-6D.}
\label{fig:diversity}
\end{figure}

\subsection{Levy-10D Convergence and Diversity}

Figures~\ref{fig:conv_levy}--\ref{fig:diversity_levy} extend the convergence analysis to Levy-10D ($d{=}10$, 10 seeds). CL-min and KB remain competitive with sequential BO.

\begin{figure}[h]
\centering
\includegraphics[width=\linewidth]{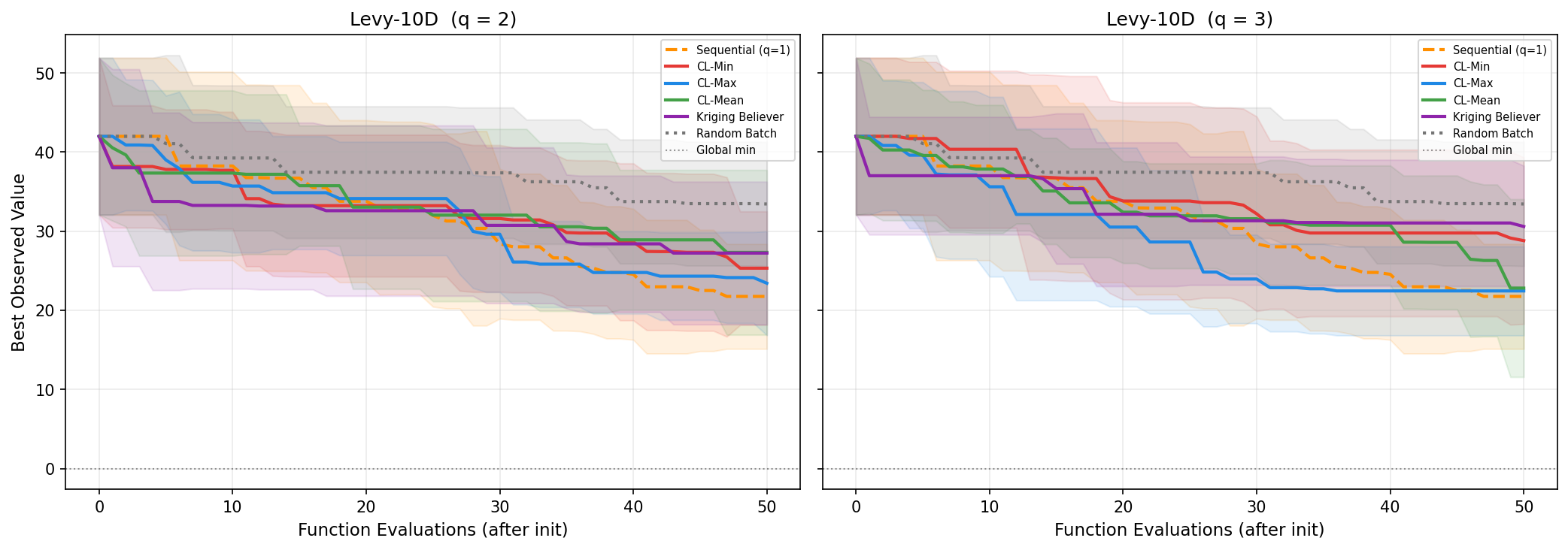}
\caption{Convergence on Levy-10D (10 seeds, $\pm$1 std).}
\label{fig:conv_levy}
\end{figure}

\begin{figure}[h]
\centering
\includegraphics[width=\linewidth]{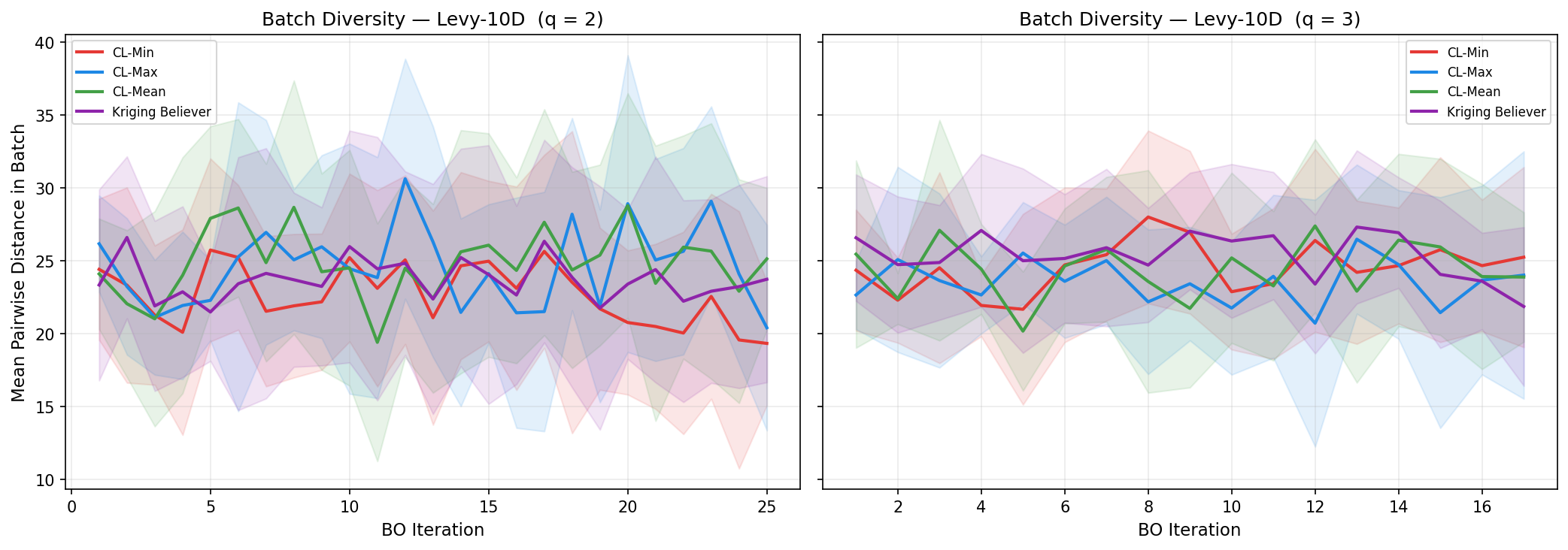}
\caption{Batch diversity on Levy-10D.}
\label{fig:diversity_levy}
\end{figure}

\subsection{SVM Tuning Convergence}

Figure~\ref{fig:svm} shows the SVM hyperparameter optimization on Breast Cancer (10 seeds). GP + CL-min achieves the lowest mean error with $2\times$ wall-clock speedup over sequential.

\begin{figure}[h]
\centering
\includegraphics[width=0.85\linewidth]{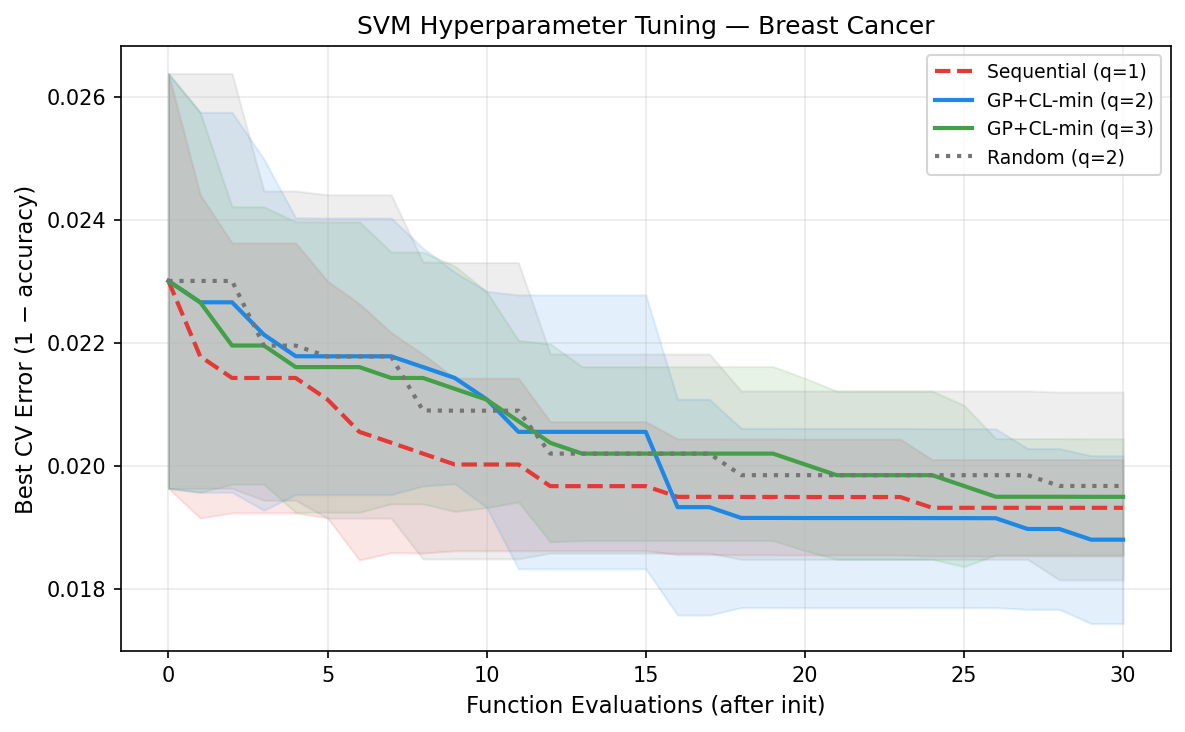}
\caption{SVM hyperparameter optimization on Breast Cancer (10 seeds).}
\label{fig:svm}
\end{figure}

\subsection{LP Comparison: Convergence Curves and Ackley-8D}

Figures~\ref{fig:lp_conv}--\ref{fig:lp_div_ackley} compare CL/KB against explicit Local Penalization on both benchmarks. CL-min's implicit penalty matches or outperforms LP in convergence and produces comparable diversity.

\begin{figure}[h]
\centering
\includegraphics[width=\linewidth]{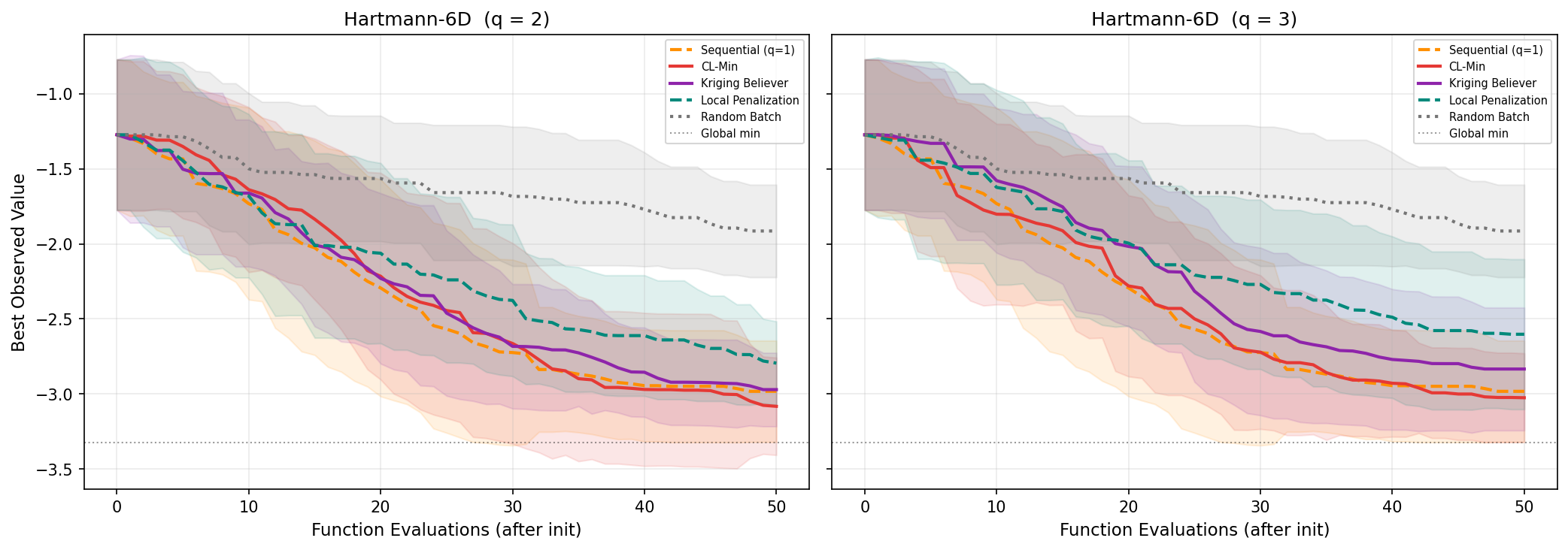}
\caption{CL/KB vs.\ LP on Hartmann-6D (20 seeds).}
\label{fig:lp_conv}
\end{figure}

\begin{figure}[h]
\centering
\includegraphics[width=\linewidth]{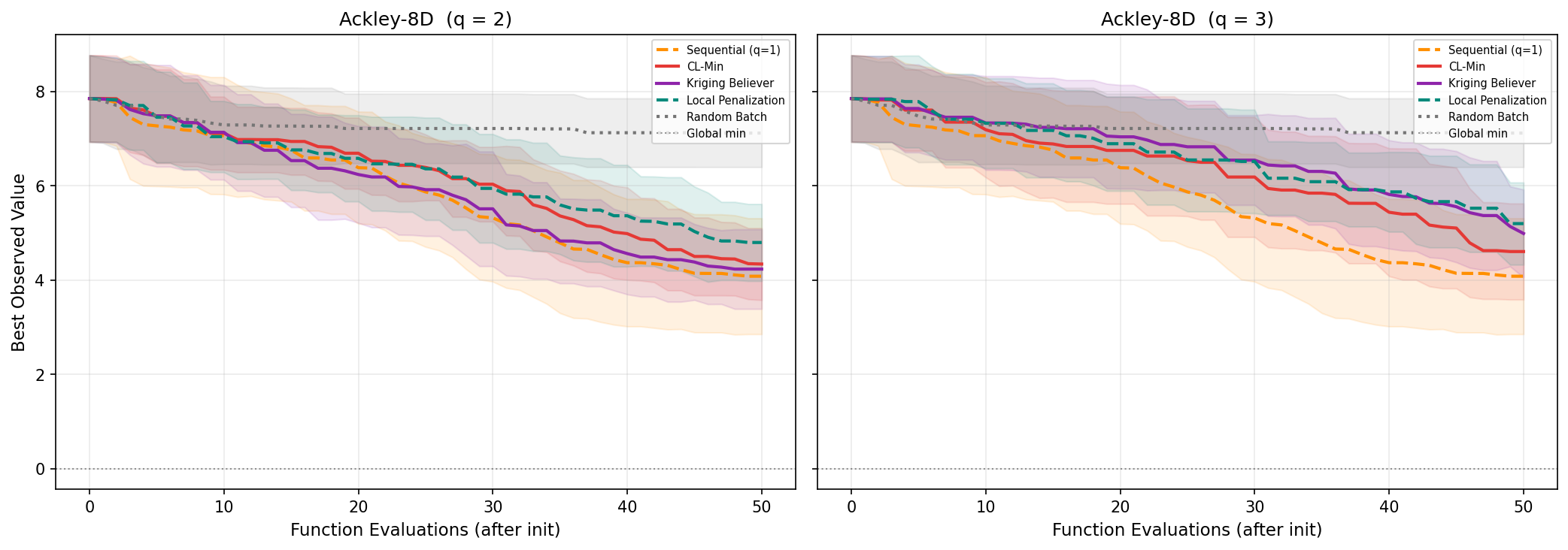}
\caption{CL/KB vs.\ LP on Ackley-8D (20 seeds).}
\label{fig:lp_ackley}
\end{figure}

\begin{figure}[h]
\centering
\includegraphics[width=\linewidth]{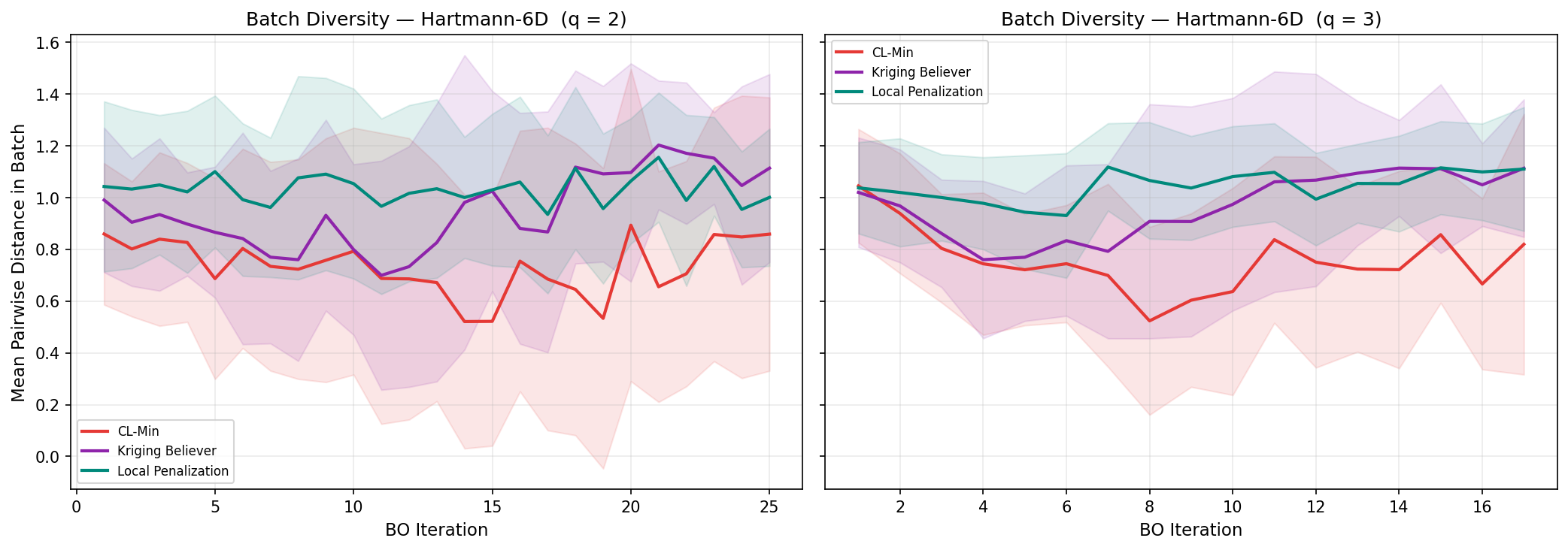}
\caption{Batch diversity: CL-min, KB, and LP on Hartmann-6D.}
\label{fig:lp_div}
\end{figure}

\begin{figure}[h]
\centering
\includegraphics[width=\linewidth]{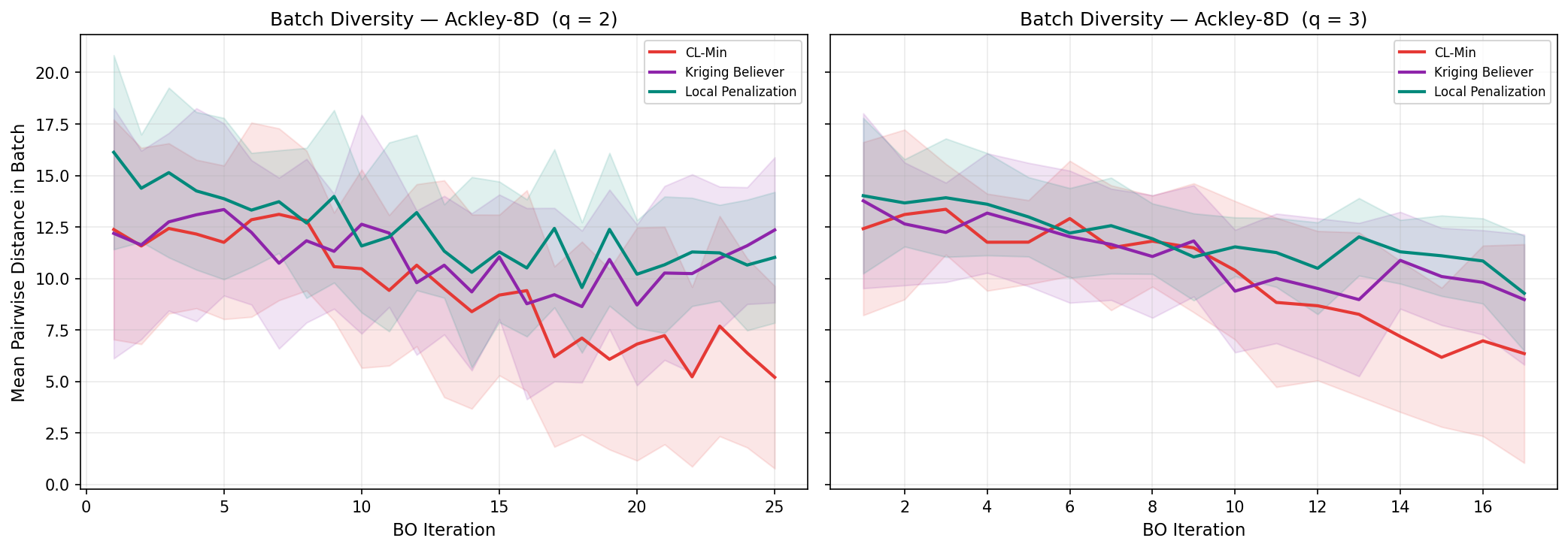}
\caption{Batch diversity: CL-min, KB, and LP on Ackley-8D.}
\label{fig:lp_div_ackley}
\end{figure}

\subsection{q-EI Comparison: Convergence Curves}

Figure~\ref{fig:qei_conv} shows that greedy CL/KB achieves convergence on par with BoTorch's joint q-EI on Hartmann-6D.

\begin{figure}[h]
\centering
\includegraphics[width=0.85\linewidth]{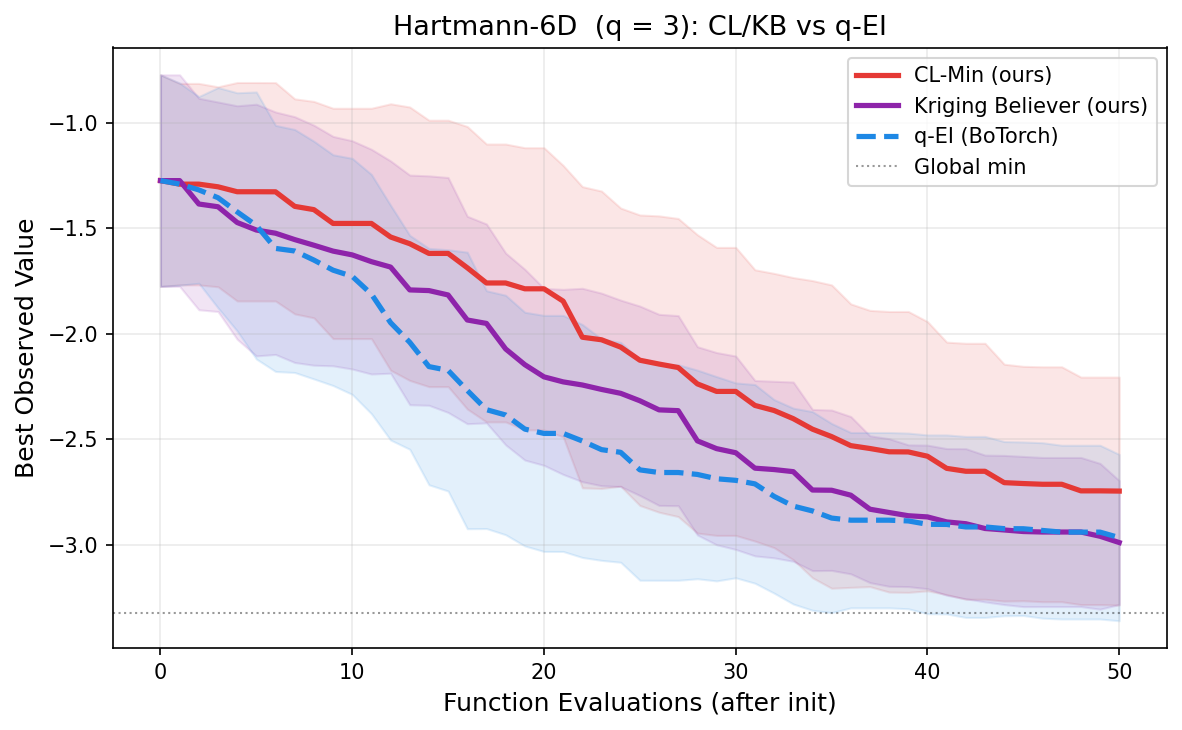}
\caption{CL/KB vs.\ q-EI (BoTorch) on Hartmann-6D ($q{=}3$, 20 seeds).}
\label{fig:qei_conv}
\end{figure}

\subsection{End-to-End Loop Timing}

\begin{table}[h]
\centering
\caption{End-to-end BO loop wall-clock (Hartmann-6D, $q{=}3$, budget${=}50$, 5 seeds).}
\label{tab:loop_timing}
\small
\begin{tabular}{lr}
\toprule
\textbf{Strategy} & \textbf{Total loop time (s)} \\
\midrule
CL-Min (ours)           & $5.34 \pm 1.19$ \\
Kriging Believer (ours) & $4.35 \pm 1.03$ \\
q-EI (BoTorch)          & $2.91 \pm 0.31$ \\
\bottomrule
\end{tabular}
\end{table}

\subsection{Retrained Models and Acquisition Agnosticism Figures}

Figure~\ref{fig:retrained} visualizes the retraining analysis from Table~\ref{tab:retrained}. Figure~\ref{fig:acq_agnostic} confirms identical SDD outcomes under EI and UCB.

\begin{figure}[h]
\centering
\includegraphics[width=\linewidth]{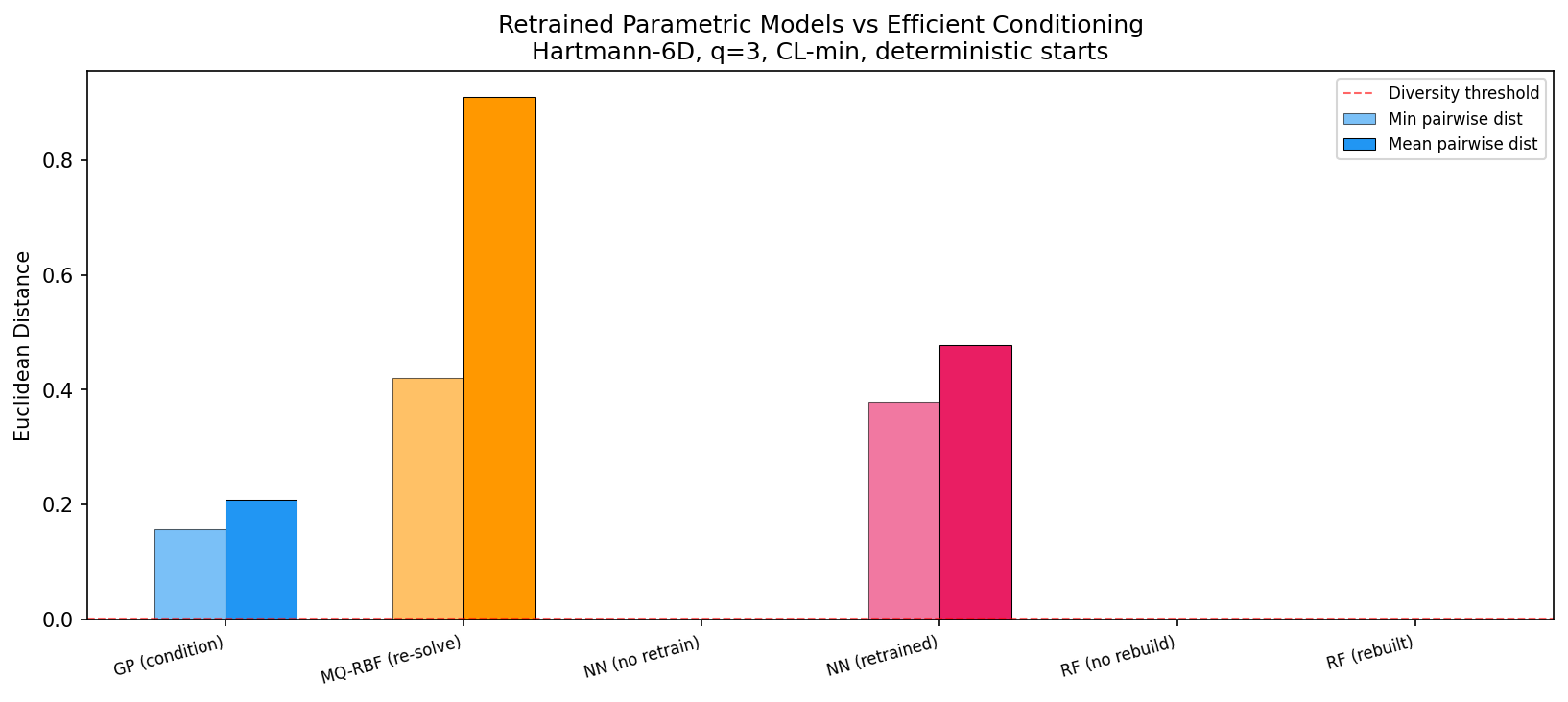}
\caption{Retrained parametric models vs.\ efficient conditioning.}
\label{fig:retrained}
\end{figure}

\begin{figure}[h]
\centering
\includegraphics[width=\linewidth]{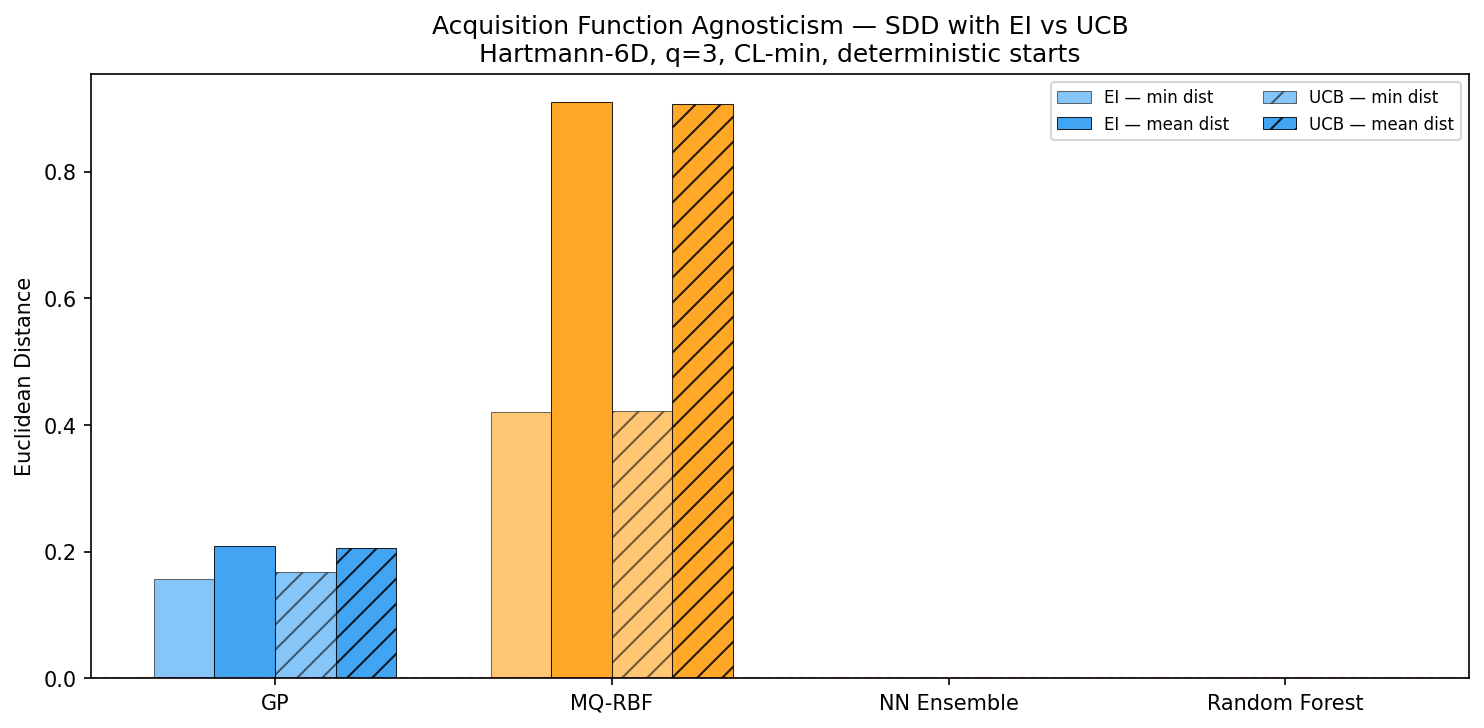}
\caption{SDD with EI vs.\ UCB: identical diversity patterns for all surrogates.}
\label{fig:acq_agnostic}
\end{figure}

\subsection{Large Batch Experiment: $q{=}10$ on Levy-10D}
\label{app:large_batch}

To test whether the diversity mechanism scales beyond $q \leq 3$, we run the SDD and full BO loops with $q{=}10$ on Levy-10D ($d{=}10$).

\textbf{SDD at $q{=}10$.} Table~\ref{tab:sdd_q10} shows that the surrogate classification is unchanged: GP and MQ-RBF produce diverse batches across all 10 members; NN and RF collapse completely, exactly as Corollary~\ref{cor:parametric_degeneracy} predicts.

\begin{table}[h]
\centering
\caption{SDD at $q{=}10$ on Levy-10D ($n_\text{init}{=}30$, fixed starts).}
\label{tab:sdd_q10}
\small
\begin{tabular}{lccl}
\toprule
\textbf{Surrogate} & \textbf{Min Dist} & \textbf{Mean Dist} & \textbf{Diverse?} \\
\midrule
GP            & 0.459 & 3.402 & Yes \\
MQ-RBF        & 3.743 & 11.712 & Yes \\
NN Ensemble   & 0.000 & 0.000 & No \\
Random Forest & 0.000 & 0.000 & No \\
\bottomrule
\end{tabular}
\end{table}

\textbf{Convergence at $q{=}10$.} Figure~\ref{fig:conv_q10} shows BO convergence with 150 evaluations (15 BO iterations) over 20 seeds. CL-min ($10.7 \pm 6.3$) and KB ($12.4 \pm 7.4$) substantially outperform random batching ($23.9 \pm 6.2$)---a $2\text{--}2.2\times$ improvement in final objective---confirming that structured batch diversity, not merely parallel evaluation, drives the optimization gains. Sequential BO ($4.7 \pm 4.3$) outperforms at equal \emph{total evaluations} because it makes 150 fully informed decisions versus 15 greedy batches of 10. However, the correct comparison for the parallel setting is at equal \emph{wall-clock time}: after 15 iterations, sequential has completed only 15 evaluations (best objective still ${\approx}35$), while CL-min has completed 150 evaluations and reached $10.7$---a $3{-}4\times$ improvement in objective value. The sample-efficiency gap at equal budgets is thus the expected cost of $10\times$ wall-clock speedup, consistent with the diminishing-returns prediction from Proposition~\ref{prop:suppression_radius}.

\begin{figure}[h]
\centering
\includegraphics[width=\linewidth]{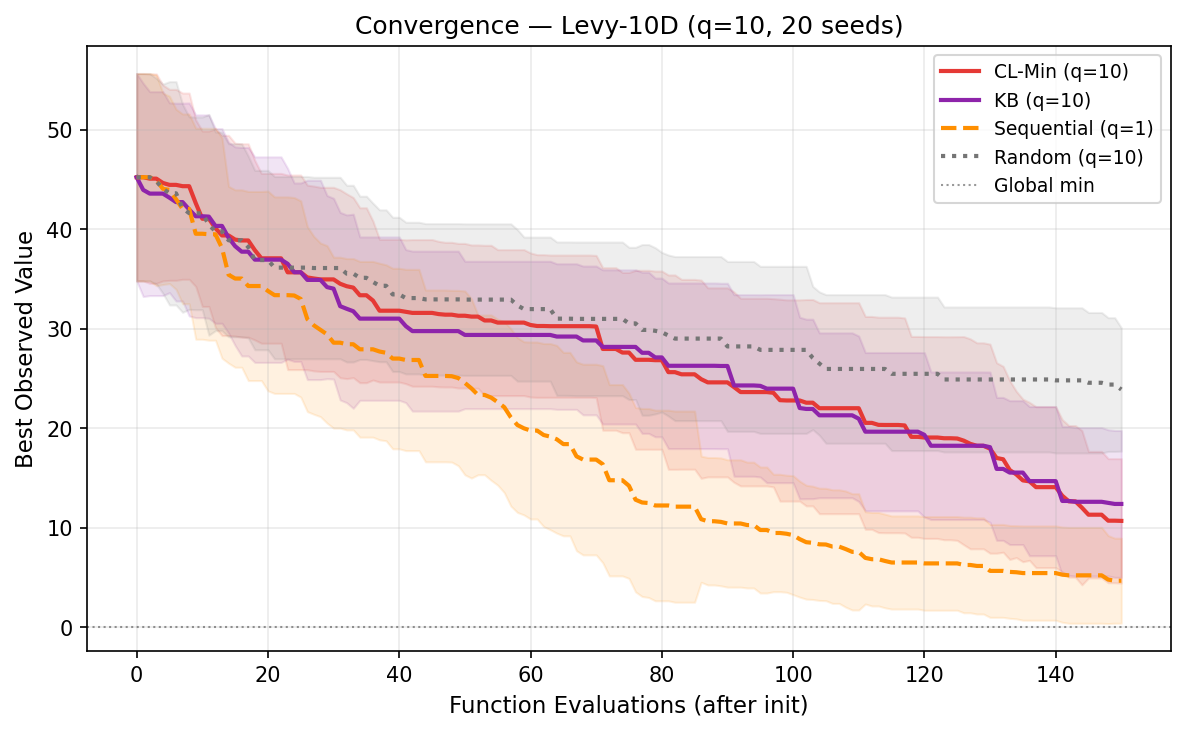}
\caption{Convergence on Levy-10D ($q{=}10$, 150 evaluations, 20 seeds, $\pm$1 std).}
\label{fig:conv_q10}
\end{figure}

\textbf{Batch diversity.} Figure~\ref{fig:div_q10} tracks mean pairwise distance within each batch across BO iterations. Diversity remains high (20--28) and never collapses, even as the surrogate accumulates data.

\begin{figure}[h]
\centering
\includegraphics[width=\linewidth]{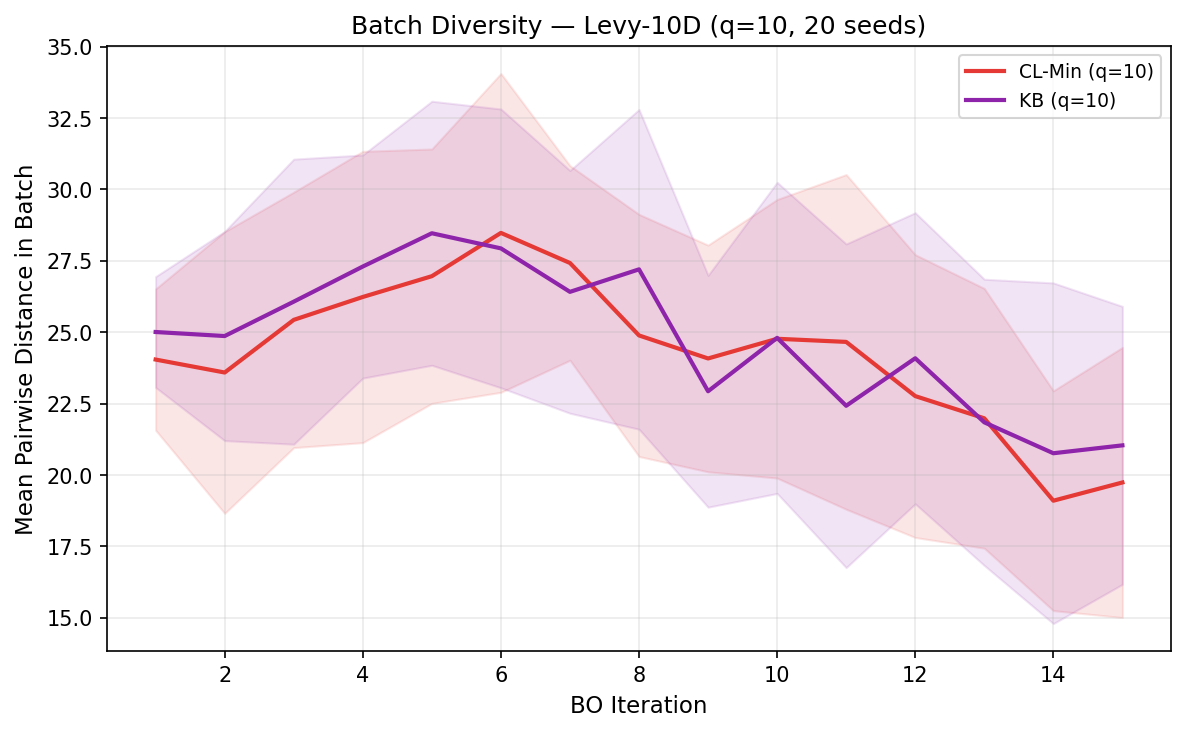}
\caption{Batch diversity across BO iterations ($q{=}10$, Levy-10D, 20 seeds).}
\label{fig:div_q10}
\end{figure}

\textbf{Per-member diversity.} Figure~\ref{fig:per_member} shows the minimum pairwise distance after adding each successive batch member. For both GP and MQ-RBF, all 10 members remain non-degenerate (min dist $\gg 0.001$). GP's minimum distance stabilizes around $0.46$ after the third member; MQ-RBF maintains higher separation ($3.7$--$7.0$) throughout.

\begin{figure}[h]
\centering
\includegraphics[width=0.85\linewidth]{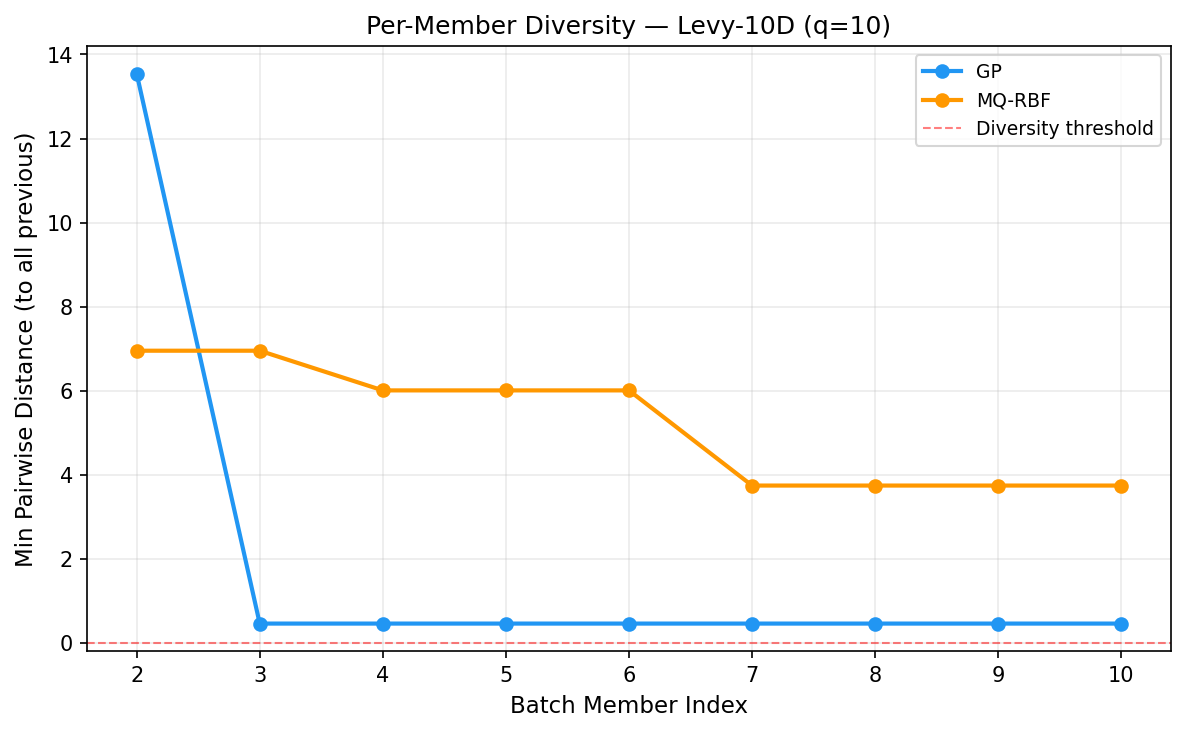}
\caption{Per-member minimum pairwise distance within a $q{=}10$ batch (Levy-10D).}
\label{fig:per_member}
\end{figure}

\FloatBarrier
\section{Hyperparameter Details}
\label{app:hyperparams}

\begin{table}[h]
\centering
\caption{Hyperparameter settings for surrogate models.}
\label{tab:hyperparams}
\small
\begin{tabular}{ll}
\toprule
\textbf{Parameter} & \textbf{Value} \\
\midrule
\multicolumn{2}{l}{\textit{Gaussian Process}} \\
\quad Kernel & Mat\'{e}rn $\nu = 2.5$ (ARD) \\
\quad Signal variance $\sigma_f^2$ & Optimized, init $= 1.0$, bounds $[10^{-3}, 10^3]$ \\
\quad Lengthscale $\ell$ & Optimized per dimension, init $= 1.0$, bounds $[10^{-2}, 10^2]$ \\
\quad Noise variance $\sigma_n^2$ & $10^{-6}$ (fixed) \\
\quad Optimizer restarts & 5 \\
\quad Normalize $y$ & Yes \\
\midrule
\multicolumn{2}{l}{\textit{MQ-RBF Network}} \\
\quad Basis function & Multiquadric: $\phi(r) = \sqrt{1 + r^2/\epsilon^2}$ \\
\quad Spread $\epsilon$ & Median pairwise distance \\
\quad Regularization & $10^{-4}$ \\
\quad Centers & Data-dependent \\
\midrule
\multicolumn{2}{l}{\textit{Neural Network Ensemble}} \\
\quad Architecture & 2 $\times$ 64 hidden, ReLU, 10 networks \\
\quad Max iterations & 2000, adaptive LR, early stopping \\
\midrule
\multicolumn{2}{l}{\textit{Random Forest}} \\
\quad Trees & 200, min leaf $= 2$ \\
\midrule
\multicolumn{2}{l}{\textit{Acquisition Optimization}} \\
\quad Function & EI ($\xi = 0.01$), L-BFGS-B, 10 restarts \\
\quad SDD fixed starts & $[0.2, 0.5, 0.8]^d$ \\
\bottomrule
\end{tabular}
\end{table}

\FloatBarrier
\section{Experimental Setup Details}
\label{app:experimental_setup}

\textbf{Benchmarks:} Hartmann-6D ($[0,1]^6$, min $\approx -3.32$), Ackley-8D ($[-5,5]^8$, min $= 0$), Levy-10D ($[-10,10]^{10}$, min $= 0$ at $\mathbf{1}$), SVM tuning (Breast Cancer, 4D, objective: $1 - \text{CV accuracy}$).

\textbf{Initialization:} Latin Hypercube with $n_\text{init} = 2d$ (except Exp~3: $n_\text{init} = 30$). \textbf{Budget:} 50 evaluations (Exps~4, 9, 10, 11); 30 (Exp~6); 150 (Exp~12). Exps~3, 7a, 7b, 8 evaluate single batches.

\textbf{Reproducibility.} All experiments are standalone Python scripts (\texttt{exp1\_illustration.py} through \texttt{exp12\_large\_batch\_q10.py}; script numbering reflects development order and may differ from paper experiment numbers) with shared \texttt{utils.py}. Fixed random seeds ensure exact reproducibility.

\FloatBarrier
\section{Compute Resources}
\label{app:compute}

All experiments ran on a single Apple M-series CPU, 16~GB RAM, no GPU.

\begin{table}[h]
\centering
\caption{Approximate runtimes.}
\label{tab:compute}
\small
\begin{tabular}{lr}
\toprule
\textbf{Experiment} & \textbf{Runtime} \\
\midrule
Exp 1--2 (1D illustrations) & $< 1$s \\
Exp 3 (SDD) & $\sim 5$s \\
Exp 4 (Convergence, Hartmann/Ackley/Levy) & $\sim 13$ min \\
Exp 5 (Timing) & $\sim 17$s \\
Exp 6 (SVM tuning) & $\sim 2$ min \\
Exp 7a/7b (Retrained + UCB) & $\sim 5$s \\
Exp 8 (Multi-seed SDD) & $\sim 4$s \\
Exp 9 (LP comparison) & $\sim 38$ min \\
Exp 10 (q-EI comparison) & $\sim 4$ min \\
Exp 11 (Noisy Hartmann-6D) & $\sim 5$ min \\
Exp 12 (Large batch $q{=}10$, Levy-10D) & $\sim 66$ min \\
\midrule
\textbf{Total} & $\sim 128$ min \\
\bottomrule
\end{tabular}
\end{table}

\FloatBarrier

\clearpage

\end{document}